\pdfoutput=1

\documentclass{article}

\usepackage[preprint,nonatbib]{neurips_2022}

\usepackage[utf8]{inputenc} % allow utf-8 input
\usepackage[T1]{fontenc}    % use 8-bit T1 fonts
\usepackage{hyperref}       % hyperlinks
\usepackage{url}            % simple URL typesetting
\usepackage{booktabs}       % professional-quality tables
\usepackage{amsfonts,bm,bbm}       % blackboard math symbols
\usepackage{nicefrac}       % compact symbols for 1/2, etc.
\usepackage{microtype}      % microtypography
\usepackage{xcolor}         % colors

\usepackage{graphicx}
\usepackage{amsthm}
\usepackage{tikz}

\usepackage{amsmath,amssymb} % define this before the line numbering.

\usepackage{mathtools}
\usepackage{algorithm}
\usepackage{algorithmic}
\usepackage{booktabs}
\usepackage{multirow}
\usepackage{wrapfig}
\usepackage{enumitem}

\def\Figref#1{Figure~\ref{#1}}

\def\Secref#1{Section~\ref{#1}}
\def\eqref#1{equation~\ref{#1}}
\def\Eqref#1{Equation~\ref{#1}}
\def\Algref#1{Algorithm~\ref{#1}}

\def\1{\bm{1}}

\newcommand{\test}{\mathcal{D_{\mathrm{test}}}}

\def\va{{\bm{a}}}

\def\vc{{\bm{c}}}

\def\vp{{\bm{p}}}

\def\vv{{\bm{v}}}
\def\vw{{\bm{w}}}
\def\vx{{\bm{x}}}

\def\mA{{\bm{A}}}
\def\mB{{\bm{B}}}
\def\mC{{\bm{C}}}
\def\mD{{\bm{D}}}

\def\mH{{\bm{H}}}
\def\mI{{\bm{I}}}

\def\mR{{\bm{R}}}
\def\mS{{\bm{S}}}

\def\mU{{\bm{U}}}

\def\mW{{\bm{W}}}
\def\mX{{\bm{X}}}

\def\mZ{{\bm{Z}}}

\DeclareMathAlphabet{\mathsfit}{\encodingdefault}{\sfdefault}{m}{sl}
\SetMathAlphabet{\mathsfit}{bold}{\encodingdefault}{\sfdefault}{bx}{n}

\def\gD{{\mathcal{D}}}

\def\gJ{{\mathcal{J}}}

\def\gL{{\mathcal{L}}}

\def\gR{{\mathcal{R}}}

\def\gX{{\mathcal{X}}}
\def\gY{{\mathcal{Y}}}

\def\sR{{\mathbb{R}}}

\newcommand{\E}{\mathbb{E}}

\newcommand{\R}{\mathbb{R}}

\newcommand\bbR{\ensuremath{\mathbb{R}}} % Real numbers
\newcommand\bbE{\ensuremath{\mathbb{E}}} % Expectation
\newcommand\bbP{\ensuremath{\mathbb{P}}} % Probability
\DeclareMathOperator*{\diag}{diag} % Diagonal matrix
\newcommand{\ep}{\epsilon} % epsilon
\newcommand{\ind}{\mathbbm{1}} % Indicator
\DeclarePairedDelimiter\abs{\lvert}{\rvert}%
\DeclarePairedDelimiterX{\norm}[1]{\lVert}{\rVert}{#1}
\DeclarePairedDelimiter{\brac}{\langle}{\rangle}

\DeclarePairedDelimiterX{\infdivx}[2]{(}{)}{%
  #1\;\delimsize\|\;#2%
}
\theoremstyle{plain}
\newtheorem{theorem}{Theorem}[section]

\newtheorem{lemma}[theorem]{Lemma}

\theoremstyle{definition}

\newtheorem{assumption}[theorem]{Assumption}
\theoremstyle{remark}

\newcommand{\our}{{FedDAR}}
\newcommand{\rebuttal}[1]{{\color{black} #1}}

\title{FedDAR: Federated Domain-Aware Representation Learning}

\author{%
  Aoxiao Zhong\thanks{Equal contribution.} \\
  Harvard University\\
  \texttt{aoxiaozhong@g.harvard.edu} \\
   \And
   Hao He\footnotemark[1]\\
   Massachusetts Institute of Technology \\
 \texttt{haohe@mit.edu} \\
   \AND
   Zhaolin Ren \\
   Harvard University \\
   \texttt{zhaolinren@g.harvard.edu} \\
   \And
   Na Li\\
   Harvard University \\
   \texttt{nali@seas.harvard.edu} \\
   \And
   Quanzheng Li \\
   Massachusetts General Hospital, Harvard Medical School \\
   \texttt{li.quanzheng@mgh.harvard.edu} \\
}

\begin{document}

%******************
\maketitle

% \begin{abstract}
% Federated learning (FL) allows many clients to collaboratively train a model without gathering data to a central node. In practice, data distribution usually varies across clients. Prior works~\cite{collins2021exploiting,arivazhagan2019federated} propose to learn a globally shared encoder and multiple predictors per client locally to overcome such heterogeneity. However, such a client-wise personalized strategy ignores the intrinsic connection between the heterogeneous data distributions across clients. In this work, we consider a more realistic setting for personalized FL, where we assume each client's data distribution is a mixture of several predefined domains. Under this setting, we theoretically show that the domain-aware personalized FL algorithm is superior to the domain-unaware algorithm which treats each client separately. We provide a novel federated learning framework, \our, that is able to learn a shared representation across domains and domain-wise personalized heads with the consideration of fairness across domains. We theoretically analyze the convergence of our algorithm and empirically demonstrate its superiority compared to multiple baselines on both synthetic and real-world datasets.
% \end{abstract}

\begin{abstract}
Cross-silo Federated learning (FL) has become a promising tool in machine learning applications for healthcare. It allows hospitals/institutions to train models with sufficient data while the data is kept private. To make sure the FL model is robust when facing heterogeneous data among FL clients, most efforts focus on personalizing models for clients. However, the latent relationships between clients' data are ignored. In this work, we focus on a special non-iid FL problem, called \emph{Domain-mixed FL}, where each client's data distribution is assumed to be a mixture of several predefined domains.  Recognizing the diversity of domains and the similarity within domains, we propose a novel method, \our, which learns a domain shared representation and domain-wise personalized prediction heads in a decoupled manner. For simplified linear regression settings, we have theoretically proved that \our~ enjoys a linear convergence rate.  For general settings, we have performed intensive empirical studies on both synthetic and real-world medical datasets which demonstrate its superiority over prior FL methods.

%This paper focuses on a special non-iid FL problem, called \emph{Domain-mixed FL}, where each client's data distribution is assumed to be a mixture of several predefined domains. Many real-world applications fall into this scenario. For example, one may perform FL on data from multiple hospitals where the data distribution heterogeneity usually comes from different fractions of sub-populations like patient groups indexed by race. Prior FL methods ignore such latent relationships between data and thus are sub-optimal in solving this domain-mixed FL problem. We illustrate this drawback both theoretically and empirically. To remedy this issue, we propose a novel method, \our, that learns a domain shared representation and domain-wise prediction heads in a decoupled manner. We theoretically prove \our~ enjoys a linear convergence rate and empirically demonstrate its superiority over prior FL methods on both synthetic and real-world medical datasets. 
\end{abstract}

\section{Introduction}
Federated learning (FL)~\cite{mcmahan2017communication} is a machine learning approach that allows many clients(e.g. mobile devices or organizations) to collaboratively train a model without sharing the data. It has great potential to resolve the dilemma in real-world machine learning applications, especially in the domain of healthcare. A robust and generalizable model in medical application usually requires a large amount of diverse data to train. However, collecting a large-scale centralized dataset could be expensive or even impractical due to the constraints from regulatory, ethical and legal challenges, data privacy and protection~\cite{rieke2020future}. 

While promising, applying FL to real-world problems has many technical challenges. One eminent challenge is data heterogeneity. Data across the clients are assumed to be independently and identically distributed (iid) by many FL algorithms. But this assumption rarely holds in the real world. It has been shown that non-iid data distributions will cause the failure of standard FL strategies such as FedAvg~\cite{jiang2019improving,sattler2020clustered,kairouz2019advances,li2020federated}. As an ideal model that can perform well on \textit{all} clients may not exist, it requires FL algorithms to personalize the model for different data distributions.

%While all of these works aim to learn a client-wise personalized model, they lack clear and realistic statistical assumptions on local data distributions. Recently,~\cite{marfoq2021federated} proposed FedEM, which assumes that the data distribution of each client is a mixture of $M$ underlying distributions. An EM-like algorithm is designed to learn a linear combination of $M$ shared component models with personalized mixture weights. Different from their assumption, we use actual known domains as the underlying distributions, without needing to estimate which domain a data point is drawn from. 
%\cite{papadaki2021federating} is most similar to ours in terms of the assumption of data distribution, but it doesn't involve any personalization. Instead they proposed FedMinMax to learn one global model that encourages fairness among domains.

%Considering the limitation of existing client-wise personalized FL approaches under our mixture of domain assumption, and inspired by the success of decoupled models on improving group-wise fairness\cite{dwork2018decoupled}, we design a novel FL framework, which can not only efficiently leverage diverse data from different domains to learn a common representation, but also provide domain-wise personalized head for fairer performance across the domains.

Prior theoretical work~\cite{marfoq2021personalized} shows that it is impossible to improve performances on all clients without making assumptions about the client's data distributions. Past works on personalized FL methods~\cite{marfoq2021personalized,sattler2020clustered,ghosh2020efficient,mansour2020three,deng2020adaptive} make their own assumptions and tailor their methods to those assumptions. In this paper, we propose a new and more realistic assumption where each client’s data distribution is a mixture of several predefined domains. We call our problem setting \emph{Domain-mixed FL}. It is inspired by the fact that the diversity of the medical data can be attributed to some known concept of domains, such as different demographic/ethnic groups of patients ~\cite{szczepura2005access,ranganathan2006exclusion,nationalstatistics}, different manufacturers or protocols/workflows of image scanners~\cite{maartensson2020reliability,ciompi2017importance}, and so on. Despite of the domain shifts between data domains, same domain at different clients are usually considered to have the same distribution. The data heterogeneity between FL clients actually comes from the distinct mixtures of diverse domains at clients. Furthermore, it is necessary to address the ubiquitous issue of domain shifts in healthcare data. For instance, different ethic groups could have significant differences in disease patterns and treatment responses ~\cite{szczepura2005access,ranganathan2006exclusion,nationalstatistics}. In addition, one ethic group could be a majority in one location/client, but a minority in another location/client; the mixture or composition of ethnicity could be different in local study cohorts. To reduce the potential bias in the FL model, we need put ethniciy related domain-wise personalization into our algorithm design.  Real world applications like this motivate us to \emph{personalize model for each domain instead of client}. 

FedEM\cite{marfoq2021personalized} and FedMinMax\cite{papadaki2021federating} makes similar assumption on data distribution as ours. However, FedEM assumes the domains are unknown and tries to learn a linear combination of several shared component models with personalized mixture weights through an EM-like algorithm. FedMinMax doesn't acknowledge the domain shift between domains and still aims to learn one shared model across domains by adapting minmax optimization to FL setting .

\textbf{Our Contributions.} We formulate the proposed problem setting, \emph{Domain-mixed FL}. Through our analysis, we find prior FL methods, both generic FL methods like FedAvg~\cite{mcmahan2017communication}, and personalized FL methods like FedRep~\cite{collins2021exploiting}, are sub-optimal under our setting. To address this issue, we propose a new algorithm, \emph{Federated Domain-Aware Representation Learning (FedDAR)}. 
FedDAR learns a shared model for all the clients but embedded with domain-wise personalized modules. The model contains two parts: an shared encoder across all domains and a multi-headed predictor whose heads are associated with domains. For an input from one specific domain, the model extracts representation via the shared encoder and then use the corresponding head to make the prediction.
\our~decouples the learning of the encoder and heads by alternating between the updates of the encoder and the heads. It allows the clients to run many local updates on the heads without overfitting on domains with limited data samples. This also leads to faster convergence and better performed model. %\lina{since we are saying: "more optimization", we should say "more optimization"--> better, faster, etc. Can we say stabilizes the training faster? What do you mean by "more optimization iterations on the heads stabilizes the training procedure"?} 
\our~also adapts different aggregation strategies for the two parts. We use a weighted average operation to aggregate the local updates for the encoder. With additional sample re-weighting, the overall training objective is equally weighted for each domain to encourage the fairness among domains. While for the heads, we propose a novel second-order aggregation algorithm to improve the optimality of aggregated heads. 

We theoretically show our method enjoys nice properties like linear convergence and small sample complexity in a linear case. Through extensive experiments on both synthetic and real-world datasets, we demonstrate that \our~significantly improves performance over the state-of-the-art personalized FL methods. To the best of our knowledge, our paper is among the
first efforts in domain-wise personalized federated learning that achieve such superior performance.

\section{Related work}

Besides the literature we have discussed above, other works on personalization and fairness in federated learning are also closely related to our work.
\paragraph{Personalized Federated Learning.}
Personalized federated learning has been studied from a variety of perspectives: i) local fine-tuning~\cite{wang2019federated,yu2020salvaging} ii) meta-learning~\cite{chen2018federated,fallah2020personalized,jiang2019improving,khodak2019adaptive} iii) local/global model interpolation~~\cite{deng2020adaptive,corinzia2019variational,mansour2020three}. iv) clustered FL that partition clients into clusters and learn optimal model for each cluster~\cite{sattler2020clustered,mansour2020three,ghosh2020efficient}. v) Multi-Task Learning(MTL)~\cite{vanhaesebrouck2017decentralized,smith2017federated,zantedeschi2020fully}~\cite{hanzely2020federated,hanzely2020lower,t2020personalized,huang2021personalized,li2021ditto} vi) local representations or heads for clients~\cite{arivazhagan2019federated, liang2020think,collins2021exploiting}. vii) personalized model through hypernetwork or super model~\cite{shamsian2021personalized,chen2021bridging,xu2022closing}. 
The personalization module in our approach is similar to vi). However, the targets we are personalizing the model for are the domains instead of clients.
 
\paragraph{Fairness in Federated Learning.}
There are two commonly used definitions of fairness in existing FL works. One is client fairness, usually formulated as \emph{client parity (CP)}, which requires clients to have similar performance. A few works~\cite{li2021ditto,li2019fair,mohri2019agnostic,yue2021gifair,zhang2020fairfl} have studied on this. Another is group fairness. In the centralized setting, the fundamental tradeoff between group fairness and accuracy has been studied~\cite{menon2018cost,wick2019unlocking, zhao2019inherent}, and various fair training algorithms have been proposed\cite{roh2020fairbatch,jiang2020identifying,zafar2017fairness,zemel2013learning,hardt2016equality}. Since the notions of group fairness is the same in FL setting, most of existing FL works adapt methods from centralized setting~\cite{zeng2021improving,du2021fairness,galvez2021enforcing,chu2021fedfair,cui2021addressing}. In this work, our method is not designed specifically for certain group fairness notions like demographic parity. Instead, we aim to achieve the best possible performance for each domain through personalization, admitting the difference between data domains. Moreover, our concept of data domains is not limited as demographic groups. It can also be applied to any other mixture of domain data, as long as our assumptions hold.

\section{Problem: Domain-mixed Federated Learning}
% \hhh{I reorganize the following 2 to 3 sections. I change the flow to the following: (1) section 3 to introduce (the minimum of) our problem setting; (2) section 4 to state limitations of directly applying prior methods; (3) Our method and analysis on it.}

\textbf{Notations.} Federated learning involves multiple clients. We denote number of clients as $n$. We use $i\in [n]\triangleq \{1,2,...,n\}$ to index each client. Client $i$ has a local data distribution $\mathcal{D}_i$ which induces a local learning objective, i.e., the expected risk $\gR_i(f)=\mathbb{E}_{(\vx_i,y_i)\sim \gD_i}[\ell(f(\vx_i),y_i)]$, where  $f: \mathcal{X}\rightarrow \mathcal{Y}$ is the model mapping the input $\vx\in \gX$ to the predicted label $f(\vx)\in \gY$ and $\ell: \gY \times \gY \rightarrow \mathbb{R}$ is a generic loss function. In real practice, client $i\in [n]$ has a finite number, say $L_i$, of data samples, i.e., $\mathcal{S}_i=\{(\vx_i^{j},y_i^{j})\}_{j=1}^{L_i}$. $L=\sum_{i=1}^n L_i$ denotes the total number of data samples.

\textbf{Problem Formulation of Domain-mixed Federated Learning.}
We introduce a new formulation of FL problem by assuming each clients' local data distribution is a weighted mixture of $M$ domain specific distributions. Specifically, we use $\{\tilde{\mathcal{D}}_m\}_{m=1}^M$ to denote data distributions from $M$ predefined domains. For client $i$, its local data distribution is $\gD_i = \sum_{m} \pi_{i,m} \tilde{\mathcal{D}}_m$ where the mixing coefficients $\pi_{i,m}$ stand for the probabilities of client $i$'s data sample coming from domain $m$.
% \lina{should we explain the meaning of ``mixing coefficients''?} 
Take medical application as an example, different hospitals are clients and different ethnic groups are domains. Each ethnic group have different health data while each hospital's data is a mix of ethnic group data. 

Further, the domains of the data samples are assumed to be known. We use a triplet of variables $(\vx,y,z)$ to represent the input features, label and domain. The goal of our problem is to learn a model $f(\vx,z)$ that can perform well in every domain, as shown by the following learning objective, 
\begin{equation}
\label{eqn:obj-origin}
\min_{f} \gR(f):=
\frac{1}{M}\sum_{m=1}^M\gR_{m}(f(\cdot,m))
\end{equation}
where $\gR_{m}(f(\cdot,m)) =\mathbb{E}_{(\vx,m)\sim \tilde{\gD}_m}[\ell(f(\vx,m),y)]$.  
Our problem focuses on the setting that each domain have a different conditioned label distribution, i.e., $P_m(y|\vx)$ is different in each domain $m$. 
% \hh{We assume the conditioned label distribution is different in different domain, i.e. $\forall m \neq m'$, $P_{m}(Y|\phi(X)) \neq P_{m'}(Y|\phi(X))$. 
% Such an assumption is often referred as \emph{concept shift} in the domain adaptation literature~\cite{kouw2018introduction}.}
\vspace{-5pt}
\subsection{Comparison with Prior Domain-unaware FL Problem Formulations}
\label{sec:problem-cmp}
Our FL problem introduces the concept of the domain and focuses on the model's performance in each domain. Many prior FL formulations does not recognize the existence of the domains. For example, the original federated learning algorithms like FedAvg~\cite{mcmahan2017communication}, FedProx~\cite{li2020federated} learn a globally shared model that via minimizing the averaged risk, i.e., $\min_{f} \frac{1}{n} \sum_i \gR_i(f)$. Some variants consider the fairness across the clients. To do so they optimize the worst client's performance, instead of the averaged performance, i.e., $\min_{f} \max_{i} \gR_i(f)$. Further, personalized FL algorithms, such as FedRep~\cite{collins2021exploiting}, customize the model's prediction for each client whose objective is $\min_{f_i: i \in [n]} \frac{1}{n} \sum_{i=1}^n \gR_i(f_i)$.

All the FL algorithms mentioned above will lead sub-optimal solutions to our problem since they do not make \emph{domain specific} predictions. We illustrate this point by the following toy example of linear regression: 
We assume the data in $m$'th domain is generated via the following procedure: $\vx \in \R^d$ is i.i.d sampled from a distribution $p(\vx)$ with mean zero and covariance $\mI_d$. The label $y \in \R$ obeys $y = \vx^\top \mB^* \vw^*_{m}$ where $\mB^* \in \R^{d \times k}$ is ground truth linear embedding shared by all domains, and $\vw^*_{m} \in \R^{k}$ is the linear head specific to domain $m$. Under this setting, $\tilde{\mathcal{D}}_m$ stands for data $(\vx,y)$ where $x\sim p(\vx)$ and $y=\vx^{\top} \mB^* \vw_m^*$. For each client, the local data $\mathcal{D}_i$ is a mix of data from different domains with mixed coefficients, i.e., $\mathcal{D}_i =\sum_m \pi_{i,m} \tilde{D}_m$. %We denote the distribution that a $(\vx,y)$ pair  $\tilde{D}_m$ to denote the distribution where

% Locally, the $i$'th client's data is a mixture of $M$ domains with coefficients $\pi_{i,m}$. Our goal is a model that regresses $y$ well in every domain.

\textbf{FedAvg:} learns a single model $\mB$ and $\vw$ across the all clients via the following objective,
% \begin{equation}
%     \min_{\mB \in \R^{d \times k},\vw \in \R^{k}} \frac{1}{2n} \sum_{i \in [n], m \in [M]} \pi_{i,m} \E_{\vx} (\vx^\top\mB^* \vw^*_{m} - \vx^\top\mB \vw )^2
% \end{equation}
% \lina{should the problem formulation has $y$?
\begin{equation}
    \min_{\mB,\vw } \frac{1}{2n} \sum_{i \in [n]} \E_{(\vx,y)\sim \mathcal{D}_{i}} (y - \vx^\top\mB \vw )^2:= \frac{1}{2n} \sum_{i \in [n]} \sum_{ m \in [M]} \pi_{i,m} \E_{(\vx,y)\sim \tilde{D}_{m}} (y - \vx^\top\mB \vw )^2
\end{equation}

% Since data $x$ has a covariance of $\mI_d$, the above objective is equivalent to $\frac{1}{n} \sum_{i \in [n], m \in [M]} \pi_{i,m} \| \mB^* \vw^*_{m} - \mB \vw \|^2_2$. This quadratic form achieves its minimum at $\mB=\mB^*$ and $\vw_{\texttt{avg}} = \frac{1}{n}\sum_{i \in [n], m \in [M]}\pi_{i,m} \vw_m^*$, with value $\frac{1}{n}\sum_{i \in [n], m \in [M]}\pi_{i,m} \|B^*(\vw^*_{m} - \vw_{\texttt{avg}})\|^2_2$. 

\textbf{FedRep:} learns shared representation $\mB$ and separated heads $\vw_i$ for each clients $i$ rather than for each domain $m$,
\begin{equation}
    \min_{\mB,\vw_1, \dots, \vw_n} \frac{1}{2n} \sum_{i \in [n]} \E_{(\vx,y)\sim \mathcal{D}_i} (y - \vx^\top \mB \vw_i )^2:= \frac{1}{2n} \sum_{i \in [n]}\sum_{m \in [M]} \pi_{i,m} \E_{(\vx,y)\sim\tilde{\mathcal{D}}_m} (y - \vx^\top \mB \vw_i )^2
\end{equation}

% Similarly the above objective is equivalent to $\frac{1}{n} \sum_{i \in [n], m \in [M]}\pi_{i,m} \| \mB^* \vw^*_{m} - \mB \vw_i \|^2_2$ whose minima is achieved at $\mB=\mB^*$ and $\vw_{\texttt{avg},i} = \sum_{m}\pi_{i,m} \vw_m^*$ with a value $\frac{1}{n}\sum_{i \in [n], m \in [M]}\pi_{i,m} \|B^*(\vw^*_{m} - \vw_{\texttt{avg},i})\|^2_2$.
%\lina{be more specific saying that in Fedrep, the decision is $w_i$ for each client, rather for each domain}

\textbf{FedDAR:} In contrast, in the linear case, our proposed method, \our, which will be introduced next, learns a shared representation $\mB$ and separate heads $\vw_m$ for each domain $m$,
\begin{equation} \label{eq:fedDAR-linear}
    \min_{\mB,\vw_1,\cdots,\vw_m} \frac{1}{2M} \sum_{i \in [n]}\sum_{m \in [M]} \frac{\pi_{i,m}}{\sum_{i'}\pi_{i',m}} \E_{(\vx,y)\sim \tilde{\mathcal{D}}_{m}} (y - \vx^\top \mB \vw_m )^2
\end{equation}

From the above formulations, we can see that FedAvd and FedRep are not able to achieve the zero error in our domain-mixed FL problem.

\section{Proposed Method: FedDAR}

To solve the Domain-mixed FL problem, we propose a new method called, \emph{Federated Domain-Aware Representation Learning} (FedDAR). In the following, we first introduce the model, learning objective and the details of the federated optimization algorithm. 

\subsection{Algorithm Overview} 
Our model is made of a shared encoder $\phi(\cdot; \bm{\theta})$ and $M$ domain specific heads $h_{m}(\cdot; \vw_{m})$ whose are parameterized by neural networks with the weights $\bm{\theta}$ and $\vw_{m}, \forall m \in [M]$. According to our problem formation in \Eqref{eqn:obj-origin}, our algorithm aims to solve the following optimization,
\begin{equation}
\label{eqn:obj}
\min_{\phi,h_1,...,h_M} \gR(\phi,h_1,...,h_M):=
\frac{1}{M}\sum_{m=1}^M\gR_{m}(h_m \circ \phi)
\end{equation}
We decouple the training between encoder and heads. Specifically, we alternates the learning between the encoder and the heads. The learning is done federatedly and has two conventional steps: (1) local updates; (2) aggregation at the server. 
% Note that since empirically operating heads more frequently than the encoder leads better performance, the objective in \Eqref{eqn:obj} is solving by the bi-level optimization, $\min_{\phi} \min_{h_1,...,h_M} \gR(\phi,h_1,...,h_M)$. 
\Algref{alg:our} shows the relevant code.

\textbf{Empirical Objectives with Re-weighting.} Empirically, the objectives are estimated via the finite data samples at each client. We use $\mathcal{S}_{i,m}$ to denote the set of samples from domain $m$ in client $i$, with $L_{i,m}:=|\mathcal{S}_{i,m}|$ denoting the sample size. Further, $L_i:=\sum_{m=1}^M L_{i,m}$ is the number of samples in client $i$ while $L_m:=\sum_{i=1}^n L_{i,m}$ is the total number of samples belonging to domain $m$ across all the clients. We denote the empirical risk at client $i$ specific to domain $m$ as $\hat{\gR}_{i,m}(h_m \circ \phi):= \frac{1}{L_{i,m}} \sum_{(\vx,y) \in \mathcal{S}_{i,m} } \ell( h_m \circ \phi(\vx), y)$. The empirical risk at client $i$ is designed as $\hat{\gR}_{i}(\phi,h_1,...,h_M)=\sum_m\frac{L_{i,m}}{L_i}u_m \hat{\gR}_{i,m}(h_m \circ \phi)$, where $u_m=\frac{L}{L_{m}M}$ re-weights the risk for each domain. Combining commonly used weighted average FL objective $\hat{\gR}(\phi,h_1,...,h_M)=\sum_{i=1}^n \frac{L_i}{L} \hat{\gR}_i(\phi,h_1,...,h_M)$, the overall empirical risk is derived as the following,  
\begin{equation}
\label{eqn:obj-erm}
\hat{\gR}(\phi,h_1,...,h_M) :=\sum_{i=1}^n \frac{L_i}{L} \hat{\gR}_i(\phi,h_1,...,h_M) = \frac{1}{M}\sum_{m=1}^M \hat{\gR}_{m}(h_m \circ \phi) 
\end{equation}
, where $\hat{\gR}_{m}(h_m \circ \phi):= \sum_{i=1}^n \frac{L_{i,m}}{L_m} \hat{\gR}_{i,m}(h_m \circ \phi)$. This is consistent with Equation \ref{eqn:obj}.

\subsection{Local Updates at Clients} 
In each communication round, clients use gradient descent methods to optimize representation $\phi(\cdot;\bm{\theta})$ and local heads $h_m(\cdot;\vw_m)$ for $m \in [M]$ alternately. We use $t$ to denote the current round. For a module $f$, $f^{t-1}$ denotes its optimized version after $t-1$ rounds. Each round has multiple gradient descent iterations. We use $f^{t,s}$ to denote the module in round $t$ after $s$ iterations. Since the updates are made locally, clients maintain their own copies of both modules, we use subscripts $i$ to index local copy at client $i$, e.g., $f^{t,s}_i$. We use $\texttt{GRD}$ to denote a generic gradient-base optimization step which takes three inputs: \emph{
objective function}, \emph{variables}, \emph{learning rate} and maps them into a new module with updated variables. For example, the vanilla gradient descent has the form $\texttt{GRD}(\gL(f_{\vw}), f_{\vw},\alpha) = f_{\vw - \alpha \nabla_{\vw} \gL(f_{\vw})}$. 

\textbf{For the heads}, client $i$ performs $\tau_h$ local gradient-based updates to obtain optimal head given the current shared encoder $\phi^{t-1}$. For $s \in [\tau_h]$, client $i$ updates via $h^{t,s}_{i,m} \leftarrow \texttt{GRD}(\hat{\gR}_{i,m}(h^{t,s-1}_{i,m} \circ \phi^{t-1}),h^{t,s-1}_{i,m},\alpha)$. \textbf{For the shared encoder}, the clients executes $\tau_\phi$ local updates. Specifically, for $s \in [\tau_\phi]$, client $i$ updates the local copy of the encoder via $\phi^{t,s}_{i} \leftarrow \texttt{GRD}(\hat{\gR}_{i}( \phi_i^{t, s-1},\{h^{t}_m\}_{m=1}^M),\phi^{t,s-1}_{i},\alpha)$. The re-weighting mentioned in last section is implemented by re-weighting each sample with $u_m$ when calculating the loss function.

%%%%%% reweighting %%%%%%
% Specifically,
% \[\hat{\gR}_{i}(\{h_{i,m}\}_{m=1}^M, \phi)=\sum_m\frac{L_{i,m}}{L_i}u_m \hat{\gR}_{i,m}(\{h_{i,m}\}_{m=1}^M, \phi) \]
% where $\hat{\gR}_{i,m}(\{h_{i,m}\}_{m=1}^M, \phi)=\frac{1}{L_{i,m}}\sum_{(x,y)\in\mathcal{S}_{i,m}}\ell(h_{i,m}\circ \phi (x),y)$ is loss function for client $i$, domain $m$ while $u_m=\frac{L}{L_{m}M}$ reweights the global objective to promote fairness across domains. It is derived as the following, 
% \begin{align*}
%     &\sum_{i}\frac{L_i}{L}\hat{\gR}_{i}(\{h_{i,m}\}_{m=1}^M, \phi)=\sum_{i}\frac{L_i}{L}\sum_m\frac{L_{i,m}}{L_i}u_m \hat{\gR}_{i,m}(\{h_{i,m}\}_{m=1}^M, \phi) \\
%     % =\frac{1}{M}\sum_{i}\sum_m\frac{L_{i,m}}{L_m} \hat{\gR}_{i,m}(\{h_{i,m}\}_{m=1}^M, \phi)
%     =\frac{1}{M} & \sum_m\sum_{i}\frac{L_{i,m}}{L_m} \hat{\gR}_{i,m}(\{h_{i,m}\}_{m=1}^M, \phi)
%     =\frac{1}{M}\sum_m\hat{\gR}_{m}(h_{m}, \phi).
% \end{align*}
% to match our desired objective in \Eqref{eqn:obj}. 
% In practice, this can be easily implemented by re-weighting each sample with $u_m$ when calculating the loss function.
%$\phi^{t,s}_{i} \leftarrow \texttt{GRD}(\sum_{m} \pi_{i,m} \gR_{i,m}(h^{t,s-1}_{i,m}, \phi^{t, s-1}),\phi^{t,s-1}_{i},\alpha)$.

\subsection{Aggregation at Server}

We introduce two strategies: (1) weighted average (WA); (2) second-order aggregation (SA).

\textbf{Weighted average} means the aggregated model parameters are the average of the local model's parameters weighted by the number of data samples. Specifically, for the shared encoder, we have $\bm{\theta}^{t} = \sum_{i=1}^{n} \frac{L_i}{L} \bm{\theta}^{t-1}$. Similarly for each head, we have  $\vw_{m}^{t} = \sum_{i=1}^{n} \frac{L_{i,m}}{L_m} \vw_{m,i}^{t-1}$.

\textbf{Second-order aggregation} is a more complex strategy. 
% It is motivated by the fact the heads are optimized to almost optima in each round. 
% Since each domain has its specific head, we perform head aggregation domain-wise independently. 
Ideally, we want the head aggregation generates the globally optimal model given a set of locally optimal model, as shown in the following,
\begin{equation}
\label{eq:headagg}
    \vw^* \in \arg\min_{\vw} \gJ(\vw) \triangleq \sum_{i=1}^n \alpha_i \gJ_i(\vw),~~~\text{given}~\vw_i^* = \arg\min_{\vw} \gR_i(\vw) \ \ \forall i \in [n].
\end{equation}
where $\gJ_i$ is $i$'th client's virtual objective, $\alpha_i := L_i / L$ is the importance of the client, $L_i$ is the number of data samples. We call $\gJ_i$ the virtual objective to distinguish it from the real learning objective $\gR_i$. The virtual objective is defined as that the local updates give the optimal solution w.r.t it. It is introduced since the local updates during two aggregated are not guaranteed to optimize the head to optimal w.r.t the real objective. For example, if each local updates is single step gradient descent with a learning rate $\eta$, i.e., $\vw^{t+1}_i = \vw^{t} - \eta \nabla_\vw \gR_i(\vw^t)$. Then the virtual objective becomes $\gJ_i(\vw) = \gR_i(\vw^t) + (\vw - \vw^t)^{\top}\nabla_\vw \gR_i(\vw^t) + \frac{1}{2\eta} \|\vw - \vw^t\|_2^2$ which satisfies $\vw^{t+1}_i \in \arg\min_{\vw} \gJ_i(\vw)$. Such a virtual objective leads the solution of problem~\ref{eq:headagg} to $\vw^* = \frac{1}{n} \sum_{i=1}^n \vw^*_i$ which is the simple averaging strategy. 

However, in real practice, the local updates is usually more complicated which makes the virtual objective closer to the true objective. We consider the case that the virtual objective is the second order Taylor expansion of the true objective, i.e., $\gJ(\vw) = \gR(\vw_t) + (\vw - \vw^t)^{\top}\nabla_\vw \gR(\vw^t) + \frac{1}{2} (\vw - \vw^t)^{\top} \mH_\gR(\vw^t) (\vw - \vw^t)$ where $\mH_\gR$ is the Hessian matrix. Then each round of local update equivalents to a Newton-like step, $\vw^{t+1}_i = \vw^t - \mH_{\gR_i}(\vw^t)^{-1} \nabla_{\vw}\gR_i(\vw^t)$. While $\vw^{t+1} = \vw^t - \mH_{\gR}(\vw^t)^{-1} \nabla_{\vw}\gR(\vw^t)$ is the desired globally optima. Leveraging the fact that, $\nabla_{\vw} \gR(\vw) = \sum_{i \in [n]} \alpha_i \nabla_{\vw} \gR_i(\vw)$ and $\mH_{\gR}(\vw) = \sum_{i \in [n]} \alpha_i \mH_{\gR_i}(\vw)$, we can get $\vw^{t+1}$ from $\vw^{t+1}_{i}$ via the following equation, which we call second-order aggregation,

\vspace{-5mm}
\rebuttal{
\begin{equation}
\label{eq:headagg-geo}
    \vw^{t+1}  = \mH_{\gR}(\vw^t)^{-1} \sum_{i \in [n]} \alpha_i \mH_{\gR_i}(\vw^t) \vw^{t+1}_i 
\end{equation}
}
Note that proposed head aggregation requires sending the Hessian matrix to the server which takes a communication cost being quadratic to the size of the weight. In real practice, the predictor head is usually small, e.g., a linear layer with hundreds of neurons. Thus it is acceptable to aggregate the Hessian matrix of the head's parameters. 

In the following, we provide two instances of our second-order aggregation with a linear head.
\vspace{-2mm}

\textbf{1. Linear Regression} where $\gR_i(\vw) = \frac{1}{L_i} \sum_{j=1}^{L_i} (\vw^{\top}\vx^{j}_i - y^{j})^2$ is quadratic itself. Thus the second order taylor expansion of the objective itself, i.e., $\gJ_i(\vw) = \gR_i(\vw)$. In this case, $\mH_{\gR_i}(\vw) = \mX_i^{\top}\mX_i$ where $\mX_i = [\vx^{1}_i, \cdots, \vx^{L_i}_i]^{\top}$ is the data matrix of client $i$.
\vspace{-2mm}

\textbf{2. Binary Classification} where $\gR_i(\vw) =  - \frac{1}{L_i} \sum_{j=1}^{L_i}  y^{j}_i \log \sigma(\vw^{\top}\vx^{j}_i) + (1 - y^{j}_i) \log (1 - \sigma(\vw^{\top}\vx^{j}_i))$. $\sigma$ is the sigmoid function. Let $\mu^{j}_i \triangleq \sigma(\vw^{\top}\vx^{j}_i)$ denote model's output. The gradient and the Hessian are, $\nabla_{\vw} \gR_i(\vw) = \frac{1}{L_i} \sum_{j} (\mu^{j}_i - y^{j}_i) \vx^{j}_i = \frac{1}{L_i} \1^{\top} \diag(\bm{\mu}_i - y_i) \mX_i^\top $ and $\mH_{\gR_i}(\vw) = \frac{1}{L_i}\mX_i^{\top}\mS\mX_i$ where $\mS \triangleq \diag(\mu^{1}_i(1-\mu^{1}_i), \cdots, \mu^{L_i}_i(1-\mu^{L_i}_i))$. Similar formulas can be derived for the multiclass classification. Please refer to the text book~\cite{murphy2022probabilistic} for the exact equations. 
\vspace{-10pt}
\paragraph{Remark.}
In practice, when the dimension of $\vw$ is larger than the number of samples of certain domain, the Hessian may have small singular values which causes numerically instability. To address this, we add an additional projection layer on top of the model's representation to reduce its dimension. 
% We denote the former variation as \our-WA~(weighted aggregation), the latter variation as \our-SA~(geometric-aware aggregation) and show results of both in the experiments section.

\subsection{Theoretical Result of FedDAR}
For a simplified linear regression setting as discussed in domain-mixed FL (\ref{eq:fedDAR-linear}) (cf. details in Appendix A), we \rebuttal{give below the sample complexity required for an adapted version of our algorithm (\Algref{algorithm:our-linear} in the appendix)  to enjoy linear convergence}. Due to the space limit, we only provide an informal statement to highlight the result. Formal statement and the proof are deferred in the appendix. 

\begin{theorem}[\rebuttal{Sample complexity of }FedDAR convergence in linear case (informal)]
\label{theorem:fed-MDR-linear-informal} Consider the linear setting for domain-mixed FL in (\ref{eq:fedDAR-linear}).  At each iteration, suppose that the number of samples used by each of $n$ clients to update the encoder, is $\tilde{\Omega}(\frac{d k^2}{n})$, and that the aggregate number of samples used in the update for the domain-specific heads, is $\tilde{\Omega}(k^2)$. Then, for a suitably chosen step-size, the distance between the encoder $\bf{B}_t$ \Algref{algorithm:our-linear} outputs and the true encoder $\bf{B}^*$ converges at a linear rate. 
\end{theorem}
\vspace{-5mm}
\paragraph{Remark.}
As our algorithm converges linearly to the true encoder, the per-iteration sample complexity of our algorithm gives a good estimate of the overall sample complexity. Since we expect the output of the encoder to be significantly lower-dimensional than the input (i.e. $k \ll d$), our result indicates that \Algref{algorithm:our-linear}'s sample complexity is dominated by $\tilde{\Omega}(\frac{d }{n})$, implying that the complexity reduces significantly as the number of clients $n$ increases. \rebuttal{Moreover, a key implication of our result is the capacity for our algorithm to accommodate data imbalance across domains. We note that our approach requires $\Omega(dk^2)$ samples per iteration for the update of the shared representation $\mB \in \mathbb{R}^{d \times k}$, whilst needing only $\Omega(k^2)$ samples per iteration for the update of each domain head. In particular, domains with more data can contribute disproportionately to the $\tilde{\Omega}(dk^2)$ samples required to learn the common representation, whilst domains with less data need only provide $\tilde{\Omega}(k^2)$ samples to update its domain head during the course of the algorithm. Whenever $k^2 \ll d$, which we believe is a reasonable assumption for many practical applications (e.g. medical imaging), the requirement of $\tilde{\Omega}(k^2)$ samples per domain is relatively mild. Conversely, forgoing the shared representation structure would require each domain to learn a separate $d$-dimensional classifier, requiring $\tilde{\Omega}(d)$ samples per domain, which can pose a challenge in problems with domain data imbalance.}

\begin{algorithm}[tb]
   \caption{\textsc{\our}}
   \label{alg:our}
\begin{algorithmic}
    \STATE {\bfseries Input:} Data $\mathcal{S}_{1:n}$; number of local updates $\tau_h$ for the heads, $\tau_\phi$ for representation; number of communication rounds $T$; learning rate $\alpha$.
    \STATE Initialize representation and heads $\phi^0,h_1^0,...,h_M^0$.
    \FOR{$t=1,2,...,T$}
    \STATE{Server sends $\phi^{t-1},h_1^{t-1},...,h_M^{t-1}$ to the $n$ clients;}
    \FOR{client $i=1,2,...,n$ {\bfseries in parallel}}
    \STATE Client $i$ initializes $h_{i,m}^{t,0}\leftarrow h_m^{t-1}, \forall m \in [M]$.
    \FOR{$s=1$ {\bfseries to} $\tau_h$}
    \STATE $h^{t,s}_{i,m} \leftarrow \texttt{GRD}(\hat{\gR}_{i,m}(h^{t,s-1}_{i,m}, \phi^{t-1}),h^{t,s-1}_{i,m},\alpha)$, for all $m \in [M]$.
    % \STATE $h_{i,m}^{t,s}\leftarrow\textsc{HeadUpdate}(\phi^{t-1},h_{i,m}^{t,s-1},\eta), \forall m \in [M]$.
    \ENDFOR
    
    \STATE Client $i$ sends updated heads $h_{i,m}^{t,\tau_h}$ and Hessians $\mH_{\gR_{i,m}}(h_{i,m}^{t,\tau_h})$ to the server.
    % for each domain $m\in [M]$ to the server;
    \ENDFOR
    \STATE Server aggregate the heads for each domain:
    \FOR{$m \in [M]$}
    \STATE $h^{t}_{m} \leftarrow \textsc{HeadAgg}(\{h_{1,m}^{t,\tau_h}, \mH_{\gR_{1,m}}(h_{1,m}^{t,\tau_h})\}_{i=1}^n)$ via \Eqref{eq:headagg-geo}.
    \ENDFOR
    \STATE Server sends $h_1^{t},...,h_M^{t}$ to the $n$ clients;
    \FOR{client $i=1,2,...,n$ {\bfseries in parallel}}
    % \STATE $\phi_i^{t}\leftarrow\textsc{RepUpdate}(\phi^{t-1},h^{t}_{1},...,h^{t}_{M},\eta)$
    
    \FOR{$s=1$ {\bfseries to} $\tau_\phi$}
    \STATE $\phi^{t,s}_i \leftarrow \texttt{GRD}(\hat{\gR}_{i,m}(h^{t}_{m}, \phi_i^{t,s-1}),\phi_i^{t,s-1},\alpha)$.
    \ENDFOR
    
    \STATE Client $i$ sends updated representation  $\phi_i^t=\phi_i^{t,\tau_\phi}$ to server.
    \ENDFOR
    \STATE Server computes the new representation via averaging $\phi^{t} \leftarrow \sum_{i=1}^n\frac{L_i}{L}\times \phi_i^{t}$.
    \ENDFOR
\end{algorithmic}
\end{algorithm}
% \section{Multi-linear regression}

\section{Experiments}

We validate our method's effectiveness on both synthetic and real datasets. We first experiment on the exact synthetic dataset described in our theoretical analysis to verify our theory. We then conduct experiments on a real dataset, FairFace~\cite{karkkainen2019fairface}, with controlled domain distributions to investigate the robustness of our algorithm under different levels of heterogeneity. Finally we compare our method with various baselines on a real federated learning benchmark, EXAM~\cite{dayan2021federated} with real-world domain distributions. We also conduct extensive ablation studies on it to discern the contribution of each component of our method. Full details of experimental settings can be found in the Appendix B.

\subsection{Synthetic Data}
We first run experiments on the linear regression problem analyzed in Appendix A. We generate (domain, data, label) samples as the following, $z_i \sim \mathcal{M}(\bm{\pi_i})$, $\vx_i \sim \mathcal{N}(0,\mI_d)$, $y_i \sim \mathcal{N}({\vw_{z_i}^*}^\top {\mB^*}^\top \vx_i,\sigma)$ where $\sigma=10^{-3}$ controls label observation errors, $\mathcal{M}(\bm{\pi_i})$ is a multinomial domain distribution with parameter $\bm{\pi_i}=[\pi_{i,1},...,\pi_{i,M}]\in \Delta^M$. The hyper-parameters of domain distributions $\bm{\pi_i}$ are drawn from a Dirichlet distribution, i.e., $\bm{\pi_i} \sim Dir(\alpha \vp)$, where $\vp \in \Delta^{M}$ is a prior domain distribution over $M$ domains, and $\alpha > 0$ is a concentration parameter controlling the heterogeneity of domain distributions among clients. The largest domain distributions heterogeneity is achieved as $\alpha \rightarrow 0$ where each client contains data only from a single randomly selected domain. On the other hand, when $\alpha \rightarrow \infty$, all clients have identical domain distributions that equal to the prior $\vp$. We generate ground-truth representation $\mB ^* \in \sR^{d\times k}$ and domain specific heads ${\vw_{m}^*}, \forall m \in [M]$ by sampling and normalizing Gaussian matrices. 

% \setlength{\intextsep}{0pt}%
% \setlength{\columnsep}{8pt}%
% \begin{wrapfigure}{r}{0.5\textwidth}
% \vspace{-8mm}
%   \begin{center}
%     \includegraphics[width=0.5\textwidth]{figures/bar_fig_new.pdf}
%   \vspace{-6mm}
%     \caption{Model performance under different number of training samples per client. Experiment sets 100 clients, 5 domains.}
%     \label{fig:synthetic_mse}
%   \end{center}
% % \caption{MSE for different methods with different numbers of local data at each client, where $d=20, k=2, n=100, M=5$.}
% \end{wrapfigure}

% \begin{figure}
% \centering
% \includegraphics[height=5cm]{figures/bar_fig1.pdf}
% \caption{MSE for different methods with different numbers of local data at each client, where $d=20, k=2, n=100, M=5$}
% \label{fig:synthetic_mse}
% \end{figure}

%\begin{figure}
%\centering
%\includegraphics[height=6.5cm]{figures/principal_angle_distance.pdf}
%\caption{Comparison of (principal angle) distances between the ground truth and learned representations by different methods. $d=10$, $k=2$, $n=100$, $M=5$, $n=10$ \zzz{need to add fedmdr without fancy aggregation and change it to average trajectories of multiple runs}}
%\label{fig:synthetic_convergence}
%\end{figure}

\Figref{fig:synthetic_mse} shows result of our experiments where we set $n=100$ clients, $M=5$ domains, feature dimension $k=2$. We varies the number of training samples per clients from $5$ to $20$. The result shows that \our-SA, achieves four orders of magnitude smaller errors than all the baselines: (1) Local-Only where each client train a model using its own data; (2) FedAvg which learns a single shared model; (3) FedRep which learns shared representation and client-specific heads. The results demonstrate that our method overcomes the heterogeneity of domain distributions across clients. \our-WA fails to converge under such setting, confirming the effectiveness of proposed second-order aggregation.

% that our proposed method, \our, can overcome the heterogeneity of domain distributions and the diversity among domains to learn a model which includes one common representation and domain-specific heads. Our method can achieve small errors no matter .  one shared global model(FedSGD), client-wise personalized model(FedRep) and models trained with local data only all fail.
\vspace{-5pt}
\subsection{Real Data with Controlled Distribution}
% We now validate our method with deep neural networks and a real face dataset with multiple domains.
% \vspace{-2mm}
\textbf{Dataset and Model.} We use FairFace~\cite{karkkainen2019fairface}, a public face image dataset containing 7 race groups which are considered as the domains. Each image is labeled with one of 9 age groups and gender. We use the age label as the target to build a multi-class age classifier. We created a FL setting via dividing training data to $n$ clients without duplication. Each client has a domain distribution $\bm{\pi_i} \sim Dir(\alpha \vp)$ sampled from a Dirichlet distribution. 
% and assign the corresponding number of images from 7 domains from the training set of FairFace without duplication. 
The total number of samples at each client $L_i=500$ is set to be the same in all experiments. We control the heterogeneity of domain distributions by altering $\alpha$. The label distributions are uniform for all the clients.

\textbf{Implementation and Evaluation.} We use Imagenet\cite{deng2009imagenet} pre-trained ResNet-34~\cite{he2016deep} for all experiments on this dataset. All the methods are trained for $T=100$ communication rounds. We use Adam optimizer with a learning rate of $1\times 10^{-4}$ for the first $60$ rounds and $1\times 10^{-5}$ for the last $40$ rounds.

Our evaluation metrics are the classification accuracy on the whole validation set of FairFace for each race group. We don't have extra local validation set to each client since we assume the data distribution within each domain is consistent across the clients. The numbers reported are the average over the final $10$ rounds of communication following the standard practice in~\cite{collins2021exploiting}.

Table \ref{table:face_age} report the results of our experiments. In general, our \our~achieved the best performance compared with the baselines.% \hhh{Maybe need some general description of the result before jumping the following points?}

\textbf{Effect of $k$.} The limitation of using \our-SA instead of \our-WA is the need of tuning the dimension of representation $k$. Figure \ref{fig:effect-k} shows results of the average domain test accuracy with different $k$. We can see that \our-SA can achieve better accuracy with a properly chosen $k$. 

\begin{figure}[t!]
  \begin{minipage}{.55\linewidth}
    \centering
      \vspace{-4mm}
        \includegraphics[width=0.9\textwidth]{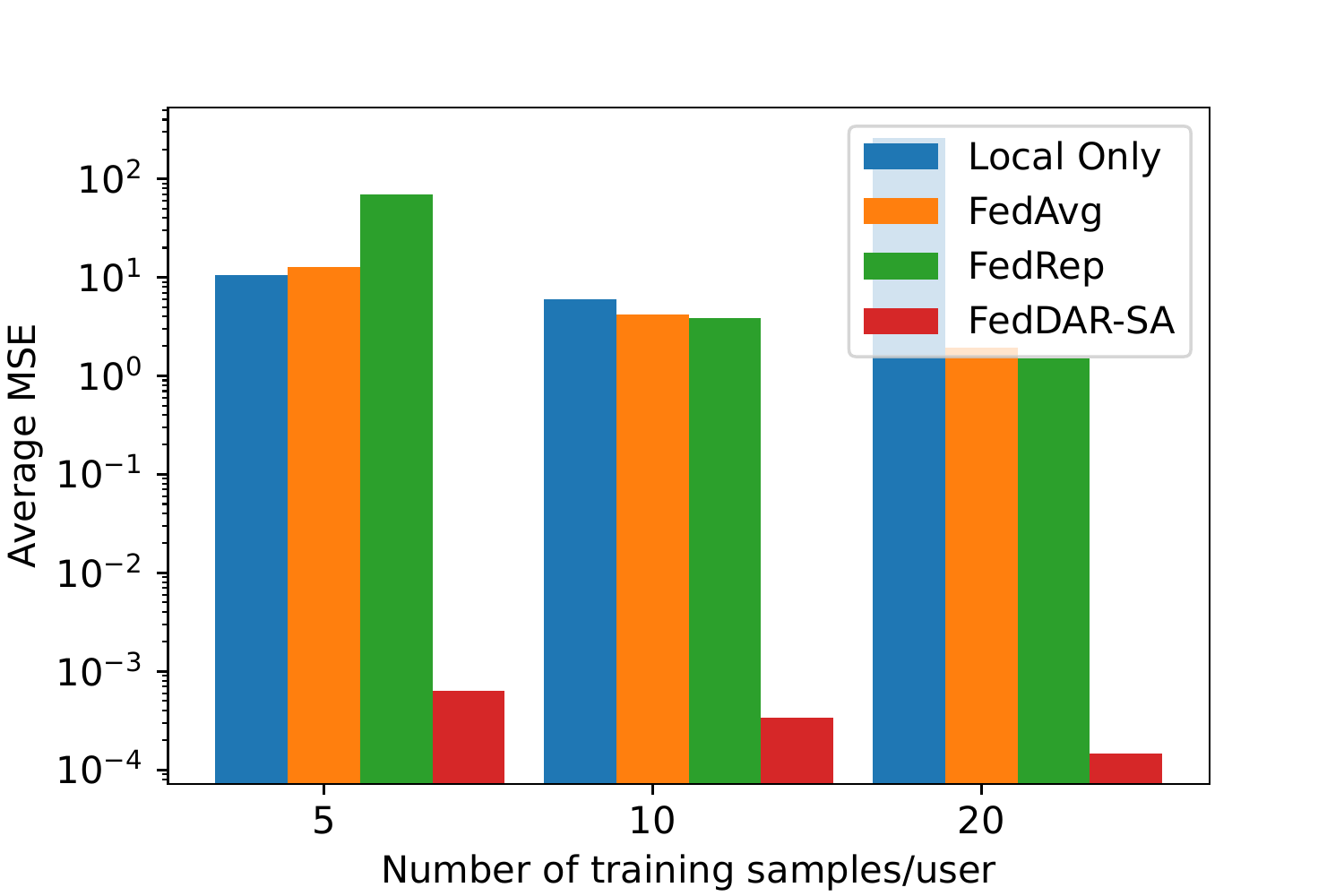}
      \vspace{-4mm}
        \caption{Performance under different number of training samples per client (100 clients, 5 domains).}
        \label{fig:synthetic_mse}        
        \vspace{-2mm}

  \end{minipage}
    \hfill
  \begin{minipage}{.43\linewidth}
    \centering
        \includegraphics[width=0.9\textwidth]{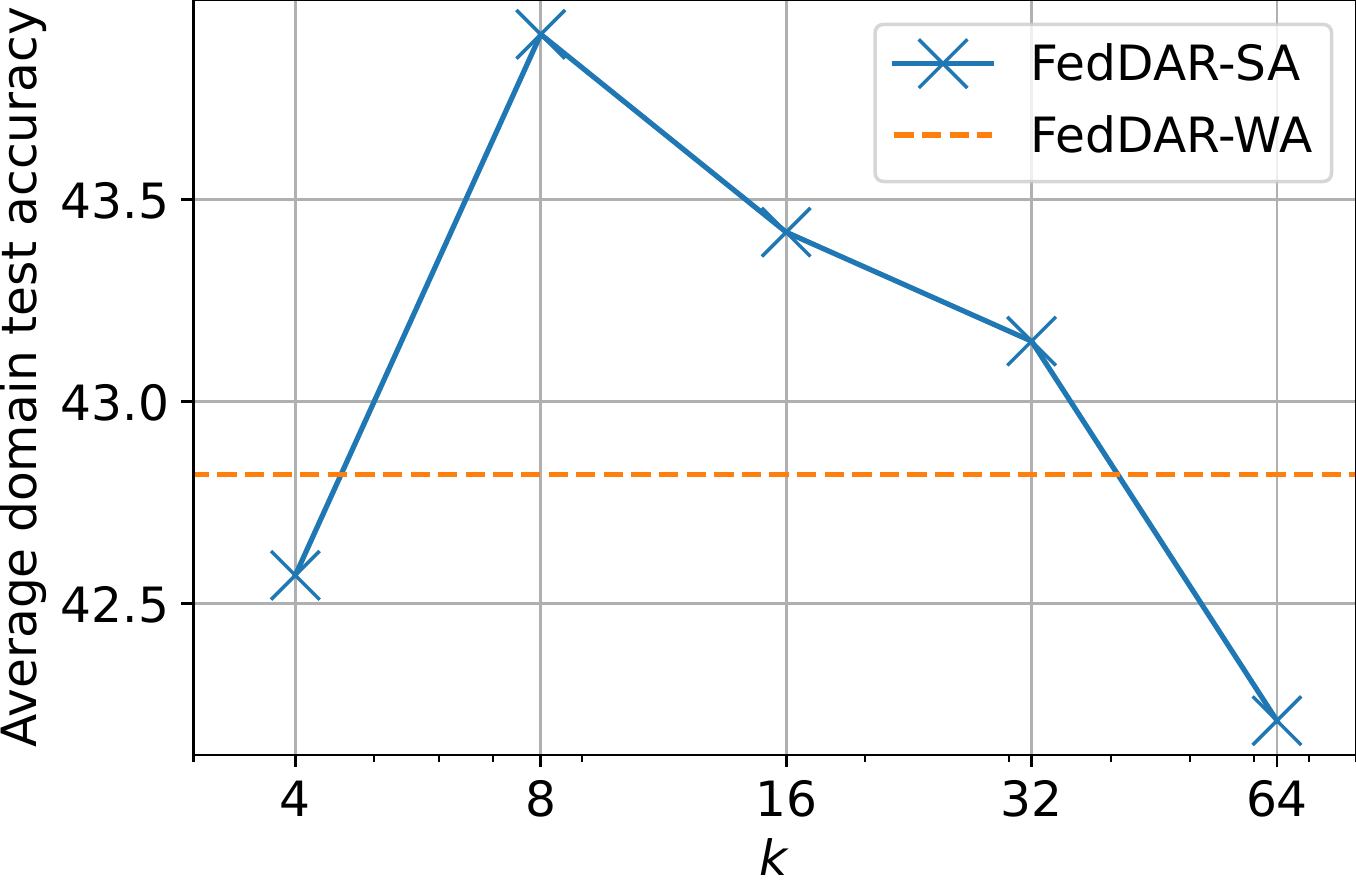}
        \vspace{-3mm}
        % \caption{Age classification accuracy (FairFace\cite{karkkainen2019fairface}) as a function of representation dimension $k$. The number of clients $n=5$. The number of domains $M=7$. The heterogeneity parameter $\alpha =1$.}
        \caption{Age classification accuracy (FairFace\cite{karkkainen2019fairface}) as a function of representation dimension $k$. (5 clients, 7 domains, heterogeneity parameter $\alpha =1$)}
        \label{fig:effect-k}
  \end{minipage}%
  \vspace{-3mm}
\end{figure}

% \begin{figure}[t]
% \centering
% \includegraphics[width=0.43\textwidth]{figures/k.pdf}
% \includegraphics[width=0.45\textwidth]{figures/k_gender.pdf}
% \vspace{-2mm}
% \caption{Effect of dimension of representation $k$ for age classification (left) and gender classification (right) on FairFace\cite{karkkainen2019fairface}, where the number of clients $n=5$, the number of domains $M=7$, the heterogeneity parameter $\alpha =1$.}
% \label{fig:effect-k}
% \end{figure}

\textbf{Robustness to Varying Levels of Heterogeneity.} From the result with various $\alpha$, we can observe that the performance of \our-SA is very stable no matter how heterogeneous the domain mixtures are. However the baselines' accuracy decrease when $\alpha$ becomes smaller.

\setlength{\tabcolsep}{4pt}
\begin{table}[t]
\caption{Min, max and average test accuracy of age classification across 7 domains ({\it race groups}) on FairFace with number of clients $n=5$, number of samples at each client $L_i=500$}
\label{table:face_age}
\centering
\resizebox{\textwidth}{!}{%
\begin{tabular}{l l ccc| ccc| ccc| ccc}
\toprule
\multirow{2}{*}{Task} & \multirow{2}{*}{Method} & \multicolumn{3}{c|}{$\alpha=0.1$} & \multicolumn{3}{c|}{$\alpha=0.5$}      & \multicolumn{3}{c|}{$\alpha=1$}      & \multicolumn{3}{c}{$\alpha=100$}              \\ 
  & & Max    & Min    & Avg    & Max           & Min  & Avg  & Max           & Min  & Avg  & Max           & Min  & Avg           \\\midrule \midrule
\multirow{4}{*}{Age} & FedAvg                        & 44.1   & 37.3   & 39.8    & 44.3 & 38.9 & 41.0 & 44.0          & 38.7 & 40.8 & 45.0          & \textbf{40.9} & 42.4          \\
& FedAvg + Multi-head           & 46.8   & 32.4   & 39.8  & \textbf{49.1}          & 34.9 & 40.0  & \textbf{51.1}          & 34.7 & 40.3 & \textbf{49.6} & 36.4 & 39.8 \\
& FedDAR-WA                     & 47.7   & 32.7   & 39.8  & 47.3 & 38.2 & 41.0 & 49.6 & 40.0 & 42.8  & 47.1          & 38.9 & 41.4          \\
& FedDAR-SA &
  \textbf{49.0} &
  \textbf{40.0} &
  \textbf{42.9} &
  47.8 &
  \textbf{40.6} &
  \textbf{42.8} &
  48.6 &
  \textbf{41.1} &
  \textbf{43.9} &
  
  48.3 & 40.2 & \textbf{42.8} \\
 %%%%% gender %%%%% 
%  \midrule
%  
%  \multirow{4}{*}{Gender} & FedAvg                   & \textbf{92.0}     & 71.7   %&83.9    &91.0    &77.4  &84.2 &90.5  &76.8  &84.5  &90.5  &77.1  &84.7           %\\
%& FedAvg + Multi-head           & 89.8   &  49.7  & 77.8   &   91.3        & 77.0 & %83.6 &   90.5        & \textbf{78.2} & 84.6 & 91.1 & 77.5 & 84.5 \\
%& FedDAR-WA                     & 89.8   & 53.4   & 80.9   & 91.5 & 76.7 & 84.3 & %90.8 & \textbf{78.0} & 84.6 & 90.0         & 76.8 & 84.1          \\
%& FedDAR-GA &
%  \textbf{92.2} &
%  \textbf{73.4} &
%  \textbf{85.1} &
%   \textbf{91.4}&
%   \textbf{78.2}&
%   \textbf{85.1}&
%   \textbf{92.2}&
%  \textbf{78.0} &
%  \textbf{85.8} &
%  \textbf{92.2}
%   & \textbf{78.1} & \textbf{85.6} \\ 
  \bottomrule
\end{tabular}%
}
\vspace{-8pt}

\end{table}
\vspace{-5pt}
\setlength{\tabcolsep}{1.4pt}
\subsection{Real Data with Real-World Data Distribution}

\textbf{Dataset and Model.} We use the EXAM dataset~\cite{dayan2021federated}, a large-scale, real-world healthcare FL study. We use part of the dataset including 6 clients with a total of 7,681 cases. We use race groups as domains. The dataset is collected from suspected COVID-19 patients at the visit of emergency department~(ED), including both Chest X-ray~(CXR) and electronic medical records~(EMR). We adopt the same data preprocessing procedure and the model as~\cite{dayan2021federated}. Our task is to predict whether the patient received the oxygen therapy higher than high-flow oxygen in 72 hours which indicates severe symptoms.

\textbf{Baselines.}
We compare our \our~against various baselines including: (1) methods that learn one global model, FedAvg\cite{mcmahan2017communication}, FedProx\cite{li2020federated}, FedMinMax\cite{papadaki2021federating} along with their local fine-tuned variants; \rebuttal{(2) train $M$ separate models with FedAvg; (3) train one global model with FedAvg fisrt, then fine-tune on $M$ domains separately with FedAvg;} (4) client-wise personalized FL approaches, FedRep\cite{collins2021exploiting}, FedPer\cite{arivazhagan2019federated}, LG-Fedavg\cite{liang2020think}.

\textbf{Implementation and Evaluation.}
We apply 5-fold cross validation. All the models are trained for $T=20$ communication rounds with Adam optimizer and a learning rate of $1\times 10^{-4}$. The models are evaluated by aggregating predictions on the local validation sets then calculating the area under curve (AUC) for each domain. The average AUCs on local validation set of clients are also reported.

\textbf{Average Performance Across Domains and Clients.} Table~\ref{table:exam} shows the average of AUCs across domains and clients. We can see that our methods, both \our-WA and \our-SA, achieve significantly better performance than all the baselines under both domain-wise and client-wise metrics. The gap between our domain-wise personalized approach and other client-wise personalized baselines shows the validity of learning domain-wise personalized model facing the diversity across domains. The reason that fine-tuning methods 
induce worse result is mainly because of the imbalanced label distribution. Each local training dataset doesn't have enough positive cases to do proper fine-tuning.

\textbf{Fairness Across Domains.} The AUCs of each specific domain in Table~\ref{table:exam}, show that our proposed \our~ method uniformly increases the AUC for each domain. The column of the minimum AUC among domains also verifies that our method indeed improve the fairness across the domains.

\begin{wraptable}{r}{0.45\textwidth}
\vspace{-18pt}
\caption{Ablation results of different components' contribution in \our.}
\vspace{5pt}
\label{table:ablation}
\centering
\resizebox{0.44\textwidth}{!}{%
\tiny
\begin{tabular}{cccccc|cc}
\toprule
RW & MH & \rebuttal{DI}& Alter & Proj & AGG & 
  \begin{tabular}[c]{c}Domain\\ Avg / Min\end{tabular} &
  \begin{tabular}[c]{c}Client\\ Avg\end{tabular} 
\\ \midrule \midrule
           &  &          &            &            &  N/A  & .861 / .773 & .856 \\
\checkmark &  &           &            &            &  N/A  & .881 / .824 & .873 \\ \midrule
\rebuttal{\checkmark} &  &\rebuttal{\checkmark}&            &            & \rebuttal{N/A} & \rebuttal{.880} / \rebuttal{.825} & \rebuttal{.866} \\

\checkmark & \checkmark &&            &            & WA & .885 / .834 & .870 \\
\checkmark & \checkmark &&            & \checkmark & WA & .877 / .817 & .870 \\
\checkmark & \checkmark &&            & \checkmark & SA & .878 / .826 & .871 \\ \midrule
\rebuttal{\checkmark} &  &\rebuttal{\checkmark}&  \rebuttal{\checkmark} &            & \rebuttal{N/A} & \rebuttal{.867} / \rebuttal{.806} & \rebuttal{.852} \\
\checkmark & \checkmark && \checkmark &            & WA & .912 / .872 & .898 \\
\checkmark & \checkmark && \checkmark & \checkmark & WA & \textbf{.918 / .863} & .904 \\
\checkmark & \checkmark && \checkmark & \checkmark & SA & \textbf{.919 / .868} & \textbf{.912} \\ \bottomrule
\end{tabular}%
}
\vspace{-8pt}
\end{wraptable}
\textbf{Ablation Studies.} i) \emph{re-weighting }(RW): First two rows in Table \ref{table:ablation} shows adding sample re-weighting can significantly improve the fairness across the domains. The minimum AUC among domains is improved by a large margin ($>0.05$); ii) \emph{multi-head }(MH), \rebuttal{\emph{domain as input feautre} (DI)} and \emph{\rebuttal{alternating} update }(Alter): Comparing three blocks in Table \ref{table:ablation}, we can see that adding multi-head itself cannot bring any improvement. We conjecture that alternating update prevents the overfitting of the heads with limited samples .This can also be reflected by the result in Table \ref{table:face_age}, where FedAvg+MH tends to perform badly on certain underrepresented domain especially when domain distributions are highly heterogeneous~($\alpha$ is small). \rebuttal{Meanwhile, using domain labels directly as feature input is not as good as multi-head, and not compatible with alternating update}; iii) \emph{projection }(Proj) and \emph{aggregation method }(AGG): Results in Table \ref{table:ablation} shows that using second-order aggregation with the projection of the features gives the best result.

\setlength{\tabcolsep}{4pt}
\begin{table}[t]
\caption{AUCs result on EXAM dataset with the domain being {\it race group}. Numbers are the means and standard deviations of metrics from 5-fold cross validation }
\label{table:exam}
\centering
\resizebox{\textwidth}{!}{%

\begin{tabular}{l|ccccc|cc|c}
\toprule
Methods & White       & Black       & Asian       & Latino      & Other       & Min         & Avg        & Client Avg \\ \midrule \midrule
Local  & .761$\pm$.023 & .815$\pm$.055 & .838$\pm$.039 & .889$\pm$.076 & .840$\pm$.038 & .759$\pm$.026 & .829$\pm$.032& .795$\pm$.023\\
\midrule
\rebuttal{separate FedAvg} & \rebuttal{.796$\pm$.022} & \rebuttal{.694$\pm$.015} & \rebuttal{.788$\pm$.047} & \rebuttal{.649$\pm$.133} & \rebuttal{.826$\pm$.046} & \rebuttal{.606$\pm$.080} & \rebuttal{.751$\pm$.026}& \rebuttal{.759$\pm$.027}\\
FedAvg  & .830$\pm$.027 & .854$\pm$.045 & .887$\pm$.022 & .834$\pm$.102 & .900$\pm$.038 & .773$\pm$.049 & .861$\pm$.019& .856$\pm$\rebuttal{.020}\\
FedAvg + FT  & .783$\pm$.044 & .835$\pm$.025 & .892$\pm$.015 & .817$\pm$.136 & .892$\pm$.048 & .727$\pm$.093 & .844$\pm$.024& .845$\pm$\rebuttal{.016}\\
\rebuttal{FedAvg + separate FT} & \rebuttal{.832$\pm$.032} & \rebuttal{.846$\pm$.043} & \rebuttal{.903$\pm$.025} & \rebuttal{.869$\pm$.099} & \rebuttal{.911$\pm$.026} & \rebuttal{.784$\pm$.054} & \rebuttal{.872$\pm$.017} & \rebuttal{.863$\pm$.024}\\
FedProx  & .834$\pm$.017 & .864$\pm$.056 & .903$\pm$.035 & .880$\pm$.085 & .912$\pm$.030 & .808$\pm$.030 & .879$\pm$.023& .868$\pm$\rebuttal{.012}\\
FedProx + FT & .806$\pm$.023 & .842$\pm$.049 & .910$\pm$.025 & .925$\pm$.085 & .898$\pm$.031 & .798$\pm$.025 & .876$\pm$.010& .858$\pm$\rebuttal{.014}\\
FedMinMax  & .839$\pm$.027 & .867$\pm$.054 & .894$\pm$.039 & .916$\pm$.053 & .903$\pm$.034 & .823$\pm$.032 & .884$\pm$.020& .872$\pm$\rebuttal{.016}\\\midrule
FedRep  & .837$\pm$.020 & .869$\pm$.050 & .888$\pm$.042 & .913$\pm$.083 & .910$\pm$.028 & .812$\pm$.028 & .884$\pm$.025& .867$\pm$\rebuttal{.013}\\
FedPer  & .835$\pm$.025 & .865$\pm$.073 & .909$\pm$.037 & .916$\pm$.036 & .911$\pm$.031 & .813$\pm$.047 & .887$\pm$.021& .873$\pm$\rebuttal{.011}\\
LG-FedAvg & .830$\pm$.029 & .858$\pm$.052 & .906$\pm$.032 & .902$\pm$.050 & .903$\pm$.033 & .814$\pm$.034 & .880$\pm$.019& .867$\pm$\rebuttal{.017}\\ \midrule

\our-WA       & \textbf{.884$\pm$.007} & \textbf{.896$\pm$.017} & .902$\pm$.034 & \textbf{.952$\pm$.041} & .928$\pm$.022 & \textbf{.872$\pm$.015 }& .912$\pm$.004& .898$\pm$\rebuttal{.006}\\
\our-SA & \textbf{.888$\pm$.004} & \textbf{.895$\pm$.038} & \textbf{.928$\pm$.032} & .939$\pm$.046 & \textbf{.948$\pm$.016} & \textbf{.868$\pm$.020} & \textbf{.919$\pm$.014}& \textbf{.912$\pm$.001}\\ 
\bottomrule
\end{tabular}%
}
\vspace{-12pt}

\end{table}

\section{Conclusions}
In this paper, we propose a novel domain-aware personalized federated learning framework based on the mixture of domain data distribution assumption. Our \our~approach is able to learn a global representation as well as domain-specific heads with balanced performance for each domain despite the heterogeneity of domain distributions across the clients. We provide both theoretical and empirical justification for its effectiveness. Our method is tested with face recognition task and a real-world medical imaging FL dataset, and can be easily extended to other complicated tasks like object detection and semantic segmentation due to its simplicity and flexibility. 

The limitations of our method include: i) the domain information for all samples is required to be known; ii) the heterogeneity of label distributions is not considered; \rebuttal{iii) the extra communication cost of sending Hessian matrices can be expensive, especially when output dimension is big.} We plan to address these issues in the future work. Other future research directions include further boosting the fairness across domains and study the setting where domains are structured, hierarchical, continuously indexed~\cite{wang2020continuously,nasery2021training} or  multi-dimensional (characterized by multiple factors)~\cite{wang2020continuously}. 

\clearpage
% Optionally include extra information (complete proofs, additional experiments and plots) in the appendix.
% This section will often be part of the supplemental material.

\bibliographystyle{plain}
\bibliography{references,federated-learning}

\begin{thebibliography}{10}

\bibitem{arivazhagan2019federated}
Manoj~Ghuhan Arivazhagan, Vinay Aggarwal, Aaditya~Kumar Singh, and Sunav
  Choudhary.
\newblock Federated learning with personalization layers.
\newblock {\em arXiv preprint arXiv:1912.00818}, 2019.

\bibitem{chen2018federated}
Fei Chen, Zhenhua Dong, Zhenguo Li, and Xiuqiang He.
\newblock Federated meta-learning for recommendation.
\newblock {\em arXiv preprint arXiv:1802.07876}, 2018.

\bibitem{chen2021bridging}
Hong-You Chen and Wei-Lun Chao.
\newblock On bridging generic and personalized federated learning for image
  classification.
\newblock In {\em International Conference on Learning Representations}, 2021.

\bibitem{chu2021fedfair}
Lingyang Chu, Lanjun Wang, Yanjie Dong, Jian Pei, Zirui Zhou, and Yong Zhang.
\newblock Fedfair: Training fair models in cross-silo federated learning.
\newblock {\em arXiv preprint arXiv:2109.05662}, 2021.

\bibitem{ciompi2017importance}
Francesco Ciompi, Oscar Geessink, Babak~Ehteshami Bejnordi, Gabriel~Silva
  De~Souza, Alexi Baidoshvili, Geert Litjens, Bram Van~Ginneken, Iris
  Nagtegaal, and Jeroen Van Der~Laak.
\newblock The importance of stain normalization in colorectal tissue
  classification with convolutional networks.
\newblock In {\em 2017 IEEE 14th International Symposium on Biomedical Imaging
  (ISBI 2017)}, pages 160--163. IEEE, 2017.

\bibitem{collins2021exploiting}
Liam Collins, Hamed Hassani, Aryan Mokhtari, and Sanjay Shakkottai.
\newblock Exploiting shared representations for personalized federated
  learning.
\newblock In {\em International Conference on Machine Learning}, pages
  2089--2099. PMLR, 2021.

\bibitem{corinzia2019variational}
Luca Corinzia, Ami Beuret, and Joachim~M Buhmann.
\newblock Variational federated multi-task learning.
\newblock {\em arXiv preprint arXiv:1906.06268}, 2019.

\bibitem{cui2021addressing}
Sen Cui, Weishen Pan, Jian Liang, Changshui Zhang, and Fei Wang.
\newblock Addressing algorithmic disparity and performance inconsistency in
  federated learning.
\newblock {\em Advances in Neural Information Processing Systems}, 34, 2021.

\bibitem{dayan2021federated}
Ittai Dayan, Holger~R Roth, Aoxiao Zhong, Ahmed Harouni, Amilcare Gentili,
  Anas~Z Abidin, Andrew Liu, Anthony~Beardsworth Costa, Bradford~J Wood,
  Chien-Sung Tsai, et~al.
\newblock Federated learning for predicting clinical outcomes in patients with
  covid-19.
\newblock {\em Nature medicine}, 27(10):1735--1743, 2021.

\bibitem{deng2009imagenet}
Jia Deng, Wei Dong, Richard Socher, Li-Jia Li, Kai Li, and Li~Fei-Fei.
\newblock Imagenet: A large-scale hierarchical image database.
\newblock In {\em 2009 IEEE conference on computer vision and pattern
  recognition}, pages 248--255. Ieee, 2009.

\bibitem{deng2020adaptive}
Yuyang Deng, Mohammad~Mahdi Kamani, and Mehrdad Mahdavi.
\newblock Adaptive personalized federated learning.
\newblock {\em arXiv preprint arXiv:2003.13461}, 2020.

\bibitem{du2021fairness}
Wei Du, Depeng Xu, Xintao Wu, and Hanghang Tong.
\newblock Fairness-aware agnostic federated learning.
\newblock In {\em Proceedings of the 2021 SIAM International Conference on Data
  Mining (SDM)}, pages 181--189. SIAM, 2021.

\bibitem{fallah2020personalized}
Alireza Fallah, Aryan Mokhtari, and Asuman Ozdaglar.
\newblock Personalized federated learning: A meta-learning approach.
\newblock {\em arXiv preprint arXiv:2002.07948}, 2020.

\bibitem{galvez2021enforcing}
Borja~Rodr{\'\i}guez G{\'a}lvez, Filip Granqvist, Rogier van Dalen, and Matt
  Seigel.
\newblock Enforcing fairness in private federated learning via the modified
  method of differential multipliers.
\newblock In {\em NeurIPS 2021 Workshop Privacy in Machine Learning}, 2021.

\bibitem{ganin2015unsupervised}
Yaroslav Ganin and Victor Lempitsky.
\newblock Unsupervised domain adaptation by backpropagation.
\newblock In {\em International conference on machine learning}, pages
  1180--1189. PMLR, 2015.

\bibitem{ghosh2020efficient}
Avishek Ghosh, Jichan Chung, Dong Yin, and Kannan Ramchandran.
\newblock An efficient framework for clustered federated learning.
\newblock {\em Advances in Neural Information Processing Systems},
  33:19586--19597, 2020.

\bibitem{golub2013matrix}
Gene~H Golub and Charles~F Van~Loan.
\newblock {\em Matrix computations}.
\newblock JHU press, 2013.

\bibitem{hanzely2020lower}
Filip Hanzely, Slavom{\'\i}r Hanzely, Samuel Horv{\'a}th, and Peter
  Richt{\'a}rik.
\newblock Lower bounds and optimal algorithms for personalized federated
  learning.
\newblock {\em Advances in Neural Information Processing Systems},
  33:2304--2315, 2020.

\bibitem{hanzely2020federated}
Filip Hanzely and Peter Richt{\'a}rik.
\newblock Federated learning of a mixture of global and local models.
\newblock {\em arXiv preprint arXiv:2002.05516}, 2020.

\bibitem{hardt2016equality}
Moritz Hardt, Eric Price, and Nati Srebro.
\newblock Equality of opportunity in supervised learning.
\newblock {\em Advances in neural information processing systems}, 29, 2016.

\bibitem{he2016deep}
Kaiming He, Xiangyu Zhang, Shaoqing Ren, and Jian Sun.
\newblock Deep residual learning for image recognition.
\newblock In {\em Proceedings of the IEEE conference on computer vision and
  pattern recognition}, pages 770--778, 2016.

\bibitem{huang2021personalized}
Yutao Huang, Lingyang Chu, Zirui Zhou, Lanjun Wang, Jiangchuan Liu, Jian Pei,
  and Yong Zhang.
\newblock Personalized cross-silo federated learning on non-iid data.
\newblock In {\em AAAI}, pages 7865--7873, 2021.

\bibitem{hull1994database}
Jonathan~J. Hull.
\newblock A database for handwritten text recognition research.
\newblock {\em IEEE Transactions on pattern analysis and machine intelligence},
  16(5):550--554, 1994.

\bibitem{jiang2020identifying}
Heinrich Jiang and Ofir Nachum.
\newblock Identifying and correcting label bias in machine learning.
\newblock In {\em International Conference on Artificial Intelligence and
  Statistics}, pages 702--712. PMLR, 2020.

\bibitem{jiang2019improving}
Yihan Jiang, Jakub Kone{\v{c}}n{\`y}, Keith Rush, and Sreeram Kannan.
\newblock Improving federated learning personalization via model agnostic meta
  learning.
\newblock {\em arXiv preprint arXiv:1909.12488}, 2019.

\bibitem{kairouz2019advances}
Peter Kairouz, H~Brendan McMahan, Brendan Avent, Aur{\'e}lien Bellet, Mehdi
  Bennis, Arjun~Nitin Bhagoji, Kallista Bonawitz, Zachary Charles, Graham
  Cormode, Rachel Cummings, et~al.
\newblock Advances and open problems in federated learning.
\newblock {\em arXiv preprint arXiv:1912.04977}, 2019.

\bibitem{karkkainen2019fairface}
Kimmo K{\"a}rkk{\"a}inen and Jungseock Joo.
\newblock Fairface: Face attribute dataset for balanced race, gender, and age.
\newblock {\em arXiv preprint arXiv:1908.04913}, 2019.

\bibitem{khodak2019adaptive}
Mikhail Khodak, Maria-Florina Balcan, and Ameet Talwalkar.
\newblock Adaptive gradient-based meta-learning methods.
\newblock {\em arXiv preprint arXiv:1906.02717}, 2019.

\bibitem{lecun1998gradient}
Yann LeCun, L{\'e}on Bottou, Yoshua Bengio, and Patrick Haffner.
\newblock Gradient-based learning applied to document recognition.
\newblock {\em Proceedings of the IEEE}, 86(11):2278--2324, 1998.

\bibitem{li2021ditto}
Tian Li, Shengyuan Hu, Ahmad Beirami, and Virginia Smith.
\newblock Ditto: Fair and robust federated learning through personalization.
\newblock In {\em International Conference on Machine Learning}, pages
  6357--6368. PMLR, 2021.

\bibitem{li2020federated}
Tian Li, Anit~Kumar Sahu, Ameet Talwalkar, and Virginia Smith.
\newblock Federated learning: Challenges, methods, and future directions.
\newblock {\em IEEE Signal Processing Magazine}, 37(3):50--60, 2020.

\bibitem{li2019fair}
Tian Li, Maziar Sanjabi, Ahmad Beirami, and Virginia Smith.
\newblock Fair resource allocation in federated learning.
\newblock {\em arXiv preprint arXiv:1905.10497}, 2019.

\bibitem{li2021fedbn}
Xiaoxiao Li, Meirui JIANG, Xiaofei Zhang, Michael Kamp, and Qi~Dou.
\newblock Fed{\{}bn{\}}: Federated learning on non-{\{}iid{\}} features via
  local batch normalization.
\newblock In {\em International Conference on Learning Representations}, 2021.

\bibitem{liang2020think}
Paul~Pu Liang, Terrance Liu, Liu Ziyin, Nicholas~B Allen, Randy~P Auerbach,
  David Brent, Ruslan Salakhutdinov, and Louis-Philippe Morency.
\newblock Think locally, act globally: Federated learning with local and global
  representations.
\newblock {\em arXiv preprint arXiv:2001.01523}, 2020.

\bibitem{mansour2020three}
Yishay Mansour, Mehryar Mohri, Jae Ro, and Ananda~Theertha Suresh.
\newblock Three approaches for personalization with applications to federated
  learning.
\newblock {\em arXiv preprint arXiv:2002.10619}, 2020.

\bibitem{marfoq2021personalized}
Othmane Marfoq, Giovanni Neglia, Laetitia Kameni, and Richard Vidal.
\newblock Personalized federated learning through local memorization.
\newblock {\em arXiv preprint arXiv:2111.09360}, 2021.

\bibitem{maartensson2020reliability}
Gustav M{\aa}rtensson, Daniel Ferreira, Tobias Granberg, Lena Cavallin, Ketil
  Oppedal, Alessandro Padovani, Irena Rektorova, Laura Bonanni, Matteo Pardini,
  Milica~G Kramberger, et~al.
\newblock The reliability of a deep learning model in clinical
  out-of-distribution mri data: a multicohort study.
\newblock {\em Medical Image Analysis}, 66:101714, 2020.

\bibitem{mcmahan2017communication}
Brendan McMahan, Eider Moore, Daniel Ramage, Seth Hampson, and Blaise~Aguera
  y~Arcas.
\newblock Communication-efficient learning of deep networks from decentralized
  data.
\newblock In {\em Artificial intelligence and statistics}, pages 1273--1282.
  PMLR, 2017.

\bibitem{menon2018cost}
Aditya~Krishna Menon and Robert~C Williamson.
\newblock The cost of fairness in binary classification.
\newblock In {\em Conference on Fairness, Accountability and Transparency},
  pages 107--118. PMLR, 2018.

\bibitem{mohri2019agnostic}
Mehryar Mohri, Gary Sivek, and Ananda~Theertha Suresh.
\newblock Agnostic federated learning.
\newblock In {\em International Conference on Machine Learning}, pages
  4615--4625. PMLR, 2019.

\bibitem{murphy2022probabilistic}
Kevin~P Murphy.
\newblock {\em Probabilistic machine learning: an introduction}.
\newblock MIT press, 2022.

\bibitem{nasery2021training}
Anshul Nasery, Soumyadeep Thakur, Vihari Piratla, Abir De, and Sunita Sarawagi.
\newblock Training for the future: A simple gradient interpolation loss to
  generalize along time.
\newblock {\em Advances in Neural Information Processing Systems}, 34, 2021.

\bibitem{netzer2011reading}
Yuval Netzer, Tao Wang, Adam Coates, Alessandro Bissacco, Bo~Wu, and Andrew~Y.
  Ng.
\newblock Reading digits in natural images with unsupervised feature learning.
\newblock In {\em NIPS Workshop on Deep Learning and Unsupervised Feature
  Learning 2011}, 2011.

\bibitem{nationalstatistics}
NHS.
\newblock Health survey for england - 2004, health of ethnic minorities, 2004.

\bibitem{papadaki2021federating}
Afroditi Papadaki, Natalia Martinez, Martin Bertran, Guillermo Sapiro, and
  Miguel Rodrigues.
\newblock Federating for learning group fair models.
\newblock {\em arXiv preprint arXiv:2110.01999}, 2021.

\bibitem{ranganathan2006exclusion}
Meghna Ranganathan and Raj Bhopal.
\newblock Exclusion and inclusion of nonwhite ethnic minority groups in 72
  north american and european cardiovascular cohort studies.
\newblock {\em PLoS medicine}, 3(3):e44, 2006.

\bibitem{rieke2020future}
Nicola Rieke, Jonny Hancox, Wenqi Li, Fausto Milletari, Holger~R Roth, Shadi
  Albarqouni, Spyridon Bakas, Mathieu~N Galtier, Bennett~A Landman, Klaus
  Maier-Hein, et~al.
\newblock The future of digital health with federated learning.
\newblock {\em NPJ digital medicine}, 3(1):1--7, 2020.

\bibitem{roh2020fairbatch}
Yuji Roh, Kangwook Lee, Steven~Euijong Whang, and Changho Suh.
\newblock Fairbatch: Batch selection for model fairness.
\newblock {\em arXiv preprint arXiv:2012.01696}, 2020.

\bibitem{sattler2020clustered}
Felix Sattler, Klaus-Robert M{\"u}ller, and Wojciech Samek.
\newblock Clustered federated learning: Model-agnostic distributed multitask
  optimization under privacy constraints.
\newblock {\em IEEE transactions on neural networks and learning systems},
  32(8):3710--3722, 2020.

\bibitem{shamsian2021personalized}
Aviv Shamsian, Aviv Navon, Ethan Fetaya, and Gal Chechik.
\newblock Personalized federated learning using hypernetworks.
\newblock In {\em International Conference on Machine Learning}, pages
  9489--9502. PMLR, 2021.

\bibitem{smith2017federated}
Virginia Smith, Chao-Kai Chiang, Maziar Sanjabi, and Ameet~S Talwalkar.
\newblock Federated multi-task learning.
\newblock {\em Advances in neural information processing systems}, 30, 2017.

\bibitem{szczepura2005access}
Ala Szczepura.
\newblock Access to health care for ethnic minority populations.
\newblock {\em Postgraduate medical journal}, 81(953):141--147, 2005.

\bibitem{t2020personalized}
Canh T~Dinh, Nguyen Tran, and Josh Nguyen.
\newblock Personalized federated learning with moreau envelopes.
\newblock {\em Advances in Neural Information Processing Systems},
  33:21394--21405, 2020.

\bibitem{tripuraneni2021provable}
Nilesh Tripuraneni, Chi Jin, and Michael Jordan.
\newblock Provable meta-learning of linear representations.
\newblock In {\em International Conference on Machine Learning}, pages
  10434--10443. PMLR, 2021.

\bibitem{vanhaesebrouck2017decentralized}
Paul Vanhaesebrouck, Aur{\'e}lien Bellet, and Marc Tommasi.
\newblock Decentralized collaborative learning of personalized models over
  networks.
\newblock In {\em Artificial Intelligence and Statistics}, pages 509--517.
  PMLR, 2017.

\bibitem{vershynin2018high}
Roman Vershynin.
\newblock {\em High-dimensional probability: An introduction with applications
  in data science}, volume~47.
\newblock Cambridge university press, 2018.

\bibitem{wang2020continuously}
Hao Wang, Hao He, and Dina Katabi.
\newblock Continuously indexed domain adaptation.
\newblock {\em arXiv preprint arXiv:2007.01807}, 2020.

\bibitem{wang2019federated}
Kangkang Wang, Rajiv Mathews, Chlo{\'e} Kiddon, Hubert Eichner, Fran{\c{c}}oise
  Beaufays, and Daniel Ramage.
\newblock Federated evaluation of on-device personalization.
\newblock {\em arXiv preprint arXiv:1910.10252}, 2019.

\bibitem{wick2019unlocking}
Michael Wick, Jean-Baptiste Tristan, et~al.
\newblock Unlocking fairness: a trade-off revisited.
\newblock {\em Advances in neural information processing systems}, 32, 2019.

\bibitem{xu2022closing}
An~Xu, Wenqi Li, Pengfei Guo, Dong Yang, Holger Roth, Ali Hatamizadeh, Can
  Zhao, Daguang Xu, Heng Huang, and Ziyue Xu.
\newblock Closing the generalization gap of cross-silo federated medical image
  segmentation.
\newblock {\em arXiv preprint arXiv:2203.10144}, 2022.

\bibitem{yu2020salvaging}
Tao Yu, Eugene Bagdasaryan, and Vitaly Shmatikov.
\newblock Salvaging federated learning by local adaptation.
\newblock {\em arXiv preprint arXiv:2002.04758}, 2020.

\bibitem{yue2021gifair}
Xubo Yue, Maher Nouiehed, and Raed~Al Kontar.
\newblock Gifair-fl: An approach for group and individual fairness in federated
  learning.
\newblock {\em arXiv preprint arXiv:2108.02741}, 2021.

\bibitem{zafar2017fairness}
Muhammad~Bilal Zafar, Isabel Valera, Manuel~Gomez Rogriguez, and Krishna~P
  Gummadi.
\newblock Fairness constraints: Mechanisms for fair classification.
\newblock In {\em Artificial Intelligence and Statistics}, pages 962--970.
  PMLR, 2017.

\bibitem{zantedeschi2020fully}
Valentina Zantedeschi, Aur{\'e}lien Bellet, and Marc Tommasi.
\newblock Fully decentralized joint learning of personalized models and
  collaboration graphs.
\newblock In {\em International Conference on Artificial Intelligence and
  Statistics}, pages 864--874. PMLR, 2020.

\bibitem{zemel2013learning}
Rich Zemel, Yu~Wu, Kevin Swersky, Toni Pitassi, and Cynthia Dwork.
\newblock Learning fair representations.
\newblock In {\em International conference on machine learning}, pages
  325--333. PMLR, 2013.

\bibitem{zeng2021improving}
Yuchen Zeng, Hongxu Chen, and Kangwook Lee.
\newblock Improving fairness via federated learning.
\newblock {\em arXiv preprint arXiv:2110.15545}, 2021.

\bibitem{zhang2020fairfl}
Daniel~Yue Zhang, Ziyi Kou, and Dong Wang.
\newblock Fairfl: A fair federated learning approach to reducing demographic
  bias in privacy-sensitive classification models.
\newblock In {\em 2020 IEEE International Conference on Big Data (Big Data)},
  pages 1051--1060. IEEE, 2020.

\bibitem{zhao2019inherent}
Han Zhao and Geoff Gordon.
\newblock Inherent tradeoffs in learning fair representations.
\newblock {\em Advances in neural information processing systems}, 32, 2019.

\end{thebibliography}

\newpage
\appendix
\section{\our~for Linear Representation}
\label{sec:theory-linear}
\subsection{Setup}

We retain the setup for linear regression considered at the start of \Secref{sec:problem-cmp}. We additionally define $\mW^* \triangleq [\vw_1^*,\cdots,\vw_M^*]^\top \in \bbR^{M \times k}$ as the concatenation of domain specific heads. For notational convenience, we let $(\vx_{i,m},y_{i,m})$ denote an (input, output) sample coming from client $i$ and the $m$-th domain. To measure the distance between any two matrices $\mA, \mB$ with the same dimensions, we use the \emph{principal angle distance}~\cite{golub2013matrix}, given by $\mathrm{dist}(\mA, \mB) \triangleq \norm{\mA_{\bot}^\top \mB}_2$, where $\mA_{\bot}$ denotes a matrix whose columns form a basis for the orthogonal complement of the range of $\mA$. To simplify analysis, we further make the following assumptions. 

\begin{assumption}[Sub-Gaussianilty]
For each $m \in [M]$ and $i \in [n]$, the samples $\vx_{i,m} \in \bbR^d$ are independent, mean zero, have covariance $\mI_d$, and has subgaussian norm 1, i.e. for every $\vv \in \bbR^d$, $\bbE[\exp(\vv^\top \vx_{i,m})] \leq \exp(\norm{\vv}^2/2)$.
\label{assumption:data_mean_zero_subgaussian}
\end{assumption}

% \vspace{-4mm}
\begin{assumption}[Domain diversity]
Let $\sigma_{\min,*} \triangleq \sigma_{\min}(\frac{1}{\sqrt{M}} \mW^*)$, i.e., $\sigma_{\min,*}$ is the minimum singular value of the head matrix. Then $\sigma_{\min,*} > 0$.
\label{assumption:domain-diversity}
\end{assumption}

% \vspace{-3mm}
\begin{assumption}[Ground truth normalization]
    The true domain parameters satisfy $\frac{1}{2} \sqrt{k} \leq
    \norm{\vw_{m}^*} \leq \sqrt{k}$ for each $m \in [M]$, and $\mB^*$ has orthonormal columns.
\label{assumption:incoherence_orthonormal}
\end{assumption}

All the above assumptions aim to simplify the theoretical analysis whilst only imposing mild constraints on the data distribution and the parameters of the target functions. Similar assumptions have also been adapted in prior work~\cite{collins2021exploiting}. 
% Note that Assumption \ref{assumption:domain-diversity} requires that the number of domains, $M$, is at least as large as $k$, where $k$ is the dimension of the embedding space that $\mB$ maps data points $x$ into. 
% Assumption \ref{assumption:domain-diversity} is closely related to but different from the assumption of client diversity in~\cite{collins2021exploiting}. 
% We also have the following assumption on $\mW^*$ and $\mB^*$.

\subsection{\our~Adapted to Linear Regression}

\begin{algorithm}[t]
\caption{\textsc{\our}~for linear regression}
\label{algorithm:our-linear}
\begin{algorithmic}
\STATE {\bfseries Input:} Step size $\eta$; number of rounds $T$
\STATE {\bfseries Client initialization: } each agent $i \in [n]$ collects $L^0$ samples, and sends $\mZ_i := \sum_{i=1}^{L^0} (y_{i}^{0,j})^2 \vx_{i}^{0,j}(\vx_{i}^{0,j})^\top$ to the server. 
\STATE {\bfseries Server initialization: } finds $\mU\mD\mU^\top \leftarrow \mbox{rank-k SVD}(\frac{1}{nL^0}\sum_i^n \mZ_i)$; sets $\mB^0 \leftarrow \mU.$ 
\FOR{$t = 0,1,\dots,T$}
\STATE {Server sends current $\mB^t$ to clients.}
\STATE {\bfseries Client computation for } $\mW^{t+1}$: 
\FOR{client $i \in [n]$}
\STATE Selects $L$ new samples $\{(\vx_i^j, y_i^j)\}$. 
\STATE Computes $\nabla_{\vw_{m}} f_{i,m}^t(\vw_{m},\mB^t) = \mA_{i,m}^t \vw_m - \va_{i,m}^t$ for each domain $m \in [M]$.
% \STATE Sends ($\nabla_{\vw_{m}} f_{i,m}^t(\vw_{m},\mB^t)$, $L_{i,m}^t$) back to server
\STATE Sends $(\mA_{i,m}^t, \va_{i,m}^t$, $L_{i,m}^t$) back to server.
\ENDFOR
\STATE {\bfseries Server update for $\mW^{t+1}$}: 
\STATE Server chooses $\vw_{m}^{t+1} \in \left\{\vw_{m} \in \bbR^k: \nabla_{\vw_{m}} \left( \frac{1}{\sum_i L_{i,m}^t}\sum_{i=1}^n f_{i,m}^t(\vw_{m},\mB^t)\right) = 0 \right\}$, $\forall m \in [M]$, i.e., $\vw_{m}^{t+1}$ that satisfies $(\sum_{i} \mA_{i,m}^t) \vw_{m}^{t+1} = \sum_{i} \va_{i,m}^t$.
\STATE Sends $\mW^{t+1} = [\vw_{1},\cdots,\vw_{M}]^\top \in \bbR^{M \times k}$ to clients. 
\STATE {\bfseries Client computation for $\mB^{t+1}$:} 
\FOR{client $i \in [n]$} 
\STATE Selects $L$ new samples $\{\vx_i^j,y_i^j\}$.
\STATE Computes $\nabla_{\mB} f_{i,m}^{t'}(\vw_{m}^{t+1}, \mB^t) = \mC_{i,m}^t \mB^t - \vc_{i,m}^t$ for each $m \in [M]$. 
\STATE Sends ($\nabla_{\mB} f_{i,m}^{t'}(\vw_{m}^{t+1},\mB^t)$, $L_{i,m}^{t'}$) back to server.
\ENDFOR
\STATE {\bfseries Server update for $\mB^{t+1}$:} 
\STATE Server computes $\tilde{\mB}^{t+1} \leftarrow \mB^t - \eta \frac{1}{m}\sum_{m=1}^M \frac{1}{\sum_i L_{i,m}^{t'}} \sum_{i=1}^n \nabla_{\mB} f_{i,m}^{t'}(\vw_{m}^{t+1}, B)$.
\STATE Server performs QR decomposition $\hat{\mB}^{t+1}, \mR^{t+1} = \texttt{QR}(\tilde{\mB}^{t+1})$. \\
Server updates $\mB^{t+1} \leftarrow \hat{\mB}^{t+1}$.
\ENDFOR
\end{algorithmic}
\end{algorithm}

% \vspace{-2mm}
We analyze an adapted version of our \our~algorithm. Since the linear regression problem has an analytic solution, to ease analysis, we update the heads $\{\vw_m\}_{m=1}^M$ at the server in closed form using local gradient information. Meanwhile, we update the representation $\mB$ by taking a step using the averaged local gradients. \Algref{algorithm:our-linear} shows the procedure of this adapted version.
% Since linear regression problem has a nice analytic solution
% we simplify the local updates for both representation and heads from gradient-based methods to closed form solutions. Similarly, the head aggregation step is also computed in closed form.\

\paragraph {The local objective} for $i$-th client in $m$-th domain at $t$-th iteration, $f_{i,m}^t(\vw_{m},\mB^t)$ is defined as the following,

\begin{align*}
    f_{i,m}^t(\vw_{m},\mB^t) \triangleq \frac{1}{2}\sum_{j=1}^{L_{i,m}^t} (y_{i,m}^j - \vw_{m}^\top \mB^\top \vx_{i,m}^j)^2, 
\end{align*}
where $L_{i,m}^t$ is the number of samples from domain $m$ at client $i$. We assume in each iteration the data points $\{\vx_{i,m}^j,y_{i,m}^j\}_{j \in [L_{i,m}^t]}$ are all newly sampled from the distribution. We denote $L = \sum_{m} L_{i,m}^t$. Note that since the objective function has a quadratic form, thus its gradient w.r.t either $\vw_m$ or $\mB$ has a linear form of $\mA_{i,m}\vw_m - \va_{i,n}$ or $\mC_{i,m}\mB - \vc_{i,m}$ which we write down explicitly in Appendix B. After every global update of the representation $\mB$, we apply an additional QR decomposition to normalize it to be column-wise orthogonal. 
% We note that  comes from the fresh data generated in the computation for $\mW^{t+1}$. Note that $L_{i,m}^t$ is a random variable denoting the number of datapoints from domain $m$ at time $t$, and that $\sum_{m} L_{i,m}^t = L$. 

% \vspace{-2mm}
% \paragraph{Client computation for $\mB^{t+1}$.} We note that $f_{i,m}^{t'}(\vw_{m},\mB^t)$ is defined similarly as the cost in client computation for $\mW^{t+1}$, but with different set of samples. Same goes for $L_{i,m}^{t'}$.

\subsection{Convergence Analysis}

We first present a theorem that states our adapted \our (\Algref{algorithm:our-linear}) enjoys linear convergence. The theorem is followed by multiple remarks which highlight key detailed points of our convergence result.

\begin{theorem}[Algorithm \ref{algorithm:our-linear} convergence]
\label{theorem:fed-MDR-linear}
Define $E_0 := 1 - \mathrm{dist}^2(\mB^0,\mB^*)$, $\bar{\sigma}_{\max,*} := \sigma_{\max} \left(\frac{1}{\sqrt{M}} \mW^* \right)$, $\bar{\sigma}_{\min,*} := \sigma_{\min}\left(\frac{1}{\sqrt{M}} \mW^* \right)$. Let $\kappa := \frac{\bar{\sigma}_{\max,*}}{\bar{\sigma}_{\min,*}}$. Suppose  
\begin{align}
\label{eq:requirement_on_L}
L \geq \tilde{\Omega}\left(\max \left\{\frac{dk^2 \kappa^4}{n E_0^2},\frac{k^2 \kappa^4}{E_0^2 \min_{m \in [M]}(\sum_{i=1}^m \pi_{i,m})}  \right\}\right).
\end{align}
Then, for any $T$ and any $\eta \leq 1/(4\bar{\sigma}_{\max,*}^2)$, with probability at least $1- T e^{-80}$,

\begin{align}
    \mathrm{dist}(\mB^T, \mB^*) \leq (1 - \eta E_0 \bar{\sigma}_{\min,*}^2/2)^{T/2} \mathrm{dist}(\mB^0,\mB^*).
\end{align}
\end{theorem}

% \begin{theorem}[Algorithm \ref{algorithm:our-linear} convergence]
% \label{theorem:fed-MDR-linear}
% Define $E_0 := 1 - \mathrm{dist}^2(\mB^0,\mB^*)$, $\bar{\sigma}_{\max,*} := \sigma_{\max} \left(\frac{1}{\sqrt{M}} \mW^* \right)$, $\bar{\sigma}_{\min,*} := \sigma_{\min}\left(\frac{1}{\sqrt{M}} \mW^* \right)$. Let $\kappa := \frac{\bar{\sigma}_{\max,*}}{\bar{\sigma}_{\min,*}}$. Suppose  
% $L \geq \left(\frac{400dk^2}{n c}\right) \left(\frac{1}{\min\left\{1/2,8 E_0/(25 \cdot 5 \kappa^2) \right\}}\right)^2, \label{eq:1st_requirement_on_L}
% $
% where $c > 0$ is an absolute constant. Suppose for each $m \in [M]$, there exists an absolute constant $C > 0$ such that
% \vspace{-2mm}
% \begin{align*}
% L \geq \max\left\{182 \log M, 16, \frac{200 C k^2 \log M}{\left(\min\left\{1/2,8 E_0/(25 \cdot 5 \kappa^2) \right\}\right)^2}\right\}\Bigg/(\sum_{i=1}^n \pi_{i,m}), \label{eq:2nd_requirement_on_L}
% \end{align*}
% Then, for any $T$ and any $\eta \leq 1/(4\bar{\sigma}_{\max,*}^2)$, with probability at least $1- T e^{-80}$,
% \vspace{-2mm}
% \begin{align}
%     \mathrm{dist}(\mB^T, \mB^*) \leq (1 - \eta E_0 \bar{\sigma}_{\min,*}^2/2)^{T/2} \mathrm{dist}(\mB^0,\mB^*).
% \end{align}
% \end{theorem}

\paragraph{Linear convergence speed: }The convergence of $\mB^T$ to $\mB^*$ is linear, assuming that (1) $ \sigma_{\min}(\frac{1}{\sqrt{M}}\mW^*) > 0$ and that (2) $ 1 - \eta E_0 \bar{\sigma}_{\min}^2 \in (0,1)$.
% \subsubsection{Requirements of $\sigma_{\min}(\mW^*) > 0$ and $M \geq k$: } In order for $\sigma_{\min}(\mW^*) = \lambda_{\min}((\mW^*)^\top (\mW^*)) > 0$ to hold, we require that the row space of $\mW^* \in \bbR^{M \times k}$ span $\bbR^k$. While this also requires that $M \geq k$, we show in Appendix B that in fact without loss of generality we may always assume that $M \geq k$. 
% We note that it is in fact unnecessary to consider the case when $M < k$. To see why, suppose for simplicity that the rows of $\mW^*$ are orthogonal to each other, such that its row space span $\bbR^M$. Then, we may replace $\mB^*$ with $\tilde{\mB}^* := \mB^* (\mW^*)^\top \in \bbR^{d \times M}$, and replace $\mW^*$ with $\tilde{\mW}^* := \mI_M$, such that $\tilde{\mB}^*$ now maps into $\bbR^M$, i.e. we in effect replace $k$ with $M$. Thus, without loss of generality, we may always assume that $M \geq k$.
\vspace{-2mm}
\vspace{-2mm}
\paragraph{Initialization of $\mB^0$:} For our convergence result to be meaningful, we need $\mathrm{dist}(\mB^0,\mB^*)$ to be close to 0. We show in Appendix A that our algorithm's choice of initial $\mB^0$ ensures that $\mathrm{dist}(\mB^0,\mB^*)$ is close enough to 0 whilst preserving privacy. When the number of samples is uniform across the domains, this comes only at the cost of a logarithmic increase in sample complexity.
\vspace{-2mm}
\vspace{-2mm}
\paragraph{Sample complexity:} The per-iteration sample complexity per client is $L$. We note that in the requirement for $L$ (\ref{eq:requirement_on_L}), we need that $L \geq \Omega(dk^2 \kappa^4/n)$; this comes from the updates for $\mB^t \in \bbR^{d \times k}$. While we expect that $d$ could be large, a large number of clients $n$ helps to mitigate the increase in sample complexity arising from $d$. We also need $L \geq \Omega(k^2 \kappa^4 \sum_{i=1}^m \pi_{i,m})$ for every domain $m \in [M]$; this requirement comes from the updates for $\vw_{m}^t$ for each of the $M$ domains.

\subsection{Proof of Theorem \ref{theorem:fed-MDR-linear}}
\subsubsection{Analysis of updating the head weights}
% \subsubsection{Analysis for update of $W^{t+1}$}
Since we are analyzing the update step for any iteration $t$, unless necessary we drop all $t$ superscripts. Let $L_{m} = \sum_{i=1}^n L_{i,m}$ denote the number of samples from domain $m \in [M]$ across the $n$ clients. Then, we can express $\nabla_{\vw_m} \sum_{i=1}^n f_{i,m}(\vw_m,\mB)$ as
\begin{align*}
    \nabla_{\vw_m} \sum_{i=1}^n f_{i,m}(\vw_m,\mB) = \sum_{i=1}^n \sum_{j=1}^{L_{i,m}} (\vw_m^\top\mB^\top\vx_{i,m}^j - y_{i,m}^j)\mB^\top\vx_{i,m}^j.
\end{align*}
Since 
$$y_{i,m}^j = (\vw_m^*)^\top (\mB^*)^\top\vx_{i,m}^j,$$
it follows that following Algorithm \ref{algorithm:our-linear},
\begin{align}
% \label{eq:w_m-update-equation-1}
    \underbrace{\left(\frac{1}{L_m}\sum_{i=1}^n \sum_{j=1}^{L_{i,m}} \left(\mB^\top\vx_{i,m}^j (\vx_{i,m}^j)^\top \mB\right)\right)}_{G_m} \vw_m^{t+1} = \frac{1}{L_m} \sum_{i=1}^n \sum_{j=1}^{L_{i,m}} \left(\mB^\top\vx_{i,m}^j (\vx_{i,m}^j)^\top \mB^*\right) \vw_m^*.
\end{align}
Reexpressing, assuming $G_m$ is invertible, we have
\begin{align}
% \label{eq:w_m-update-equation-1}
    \vw_m^{t+1} =\mB^\top\mB^* \vw_m^* + \left(G_m^{-1} \left(\frac{1}{L_m} \sum_{i=1}^n \sum_{j=1}^{L_{i,m}} \left(\mB^\top\vx_{i,m}^j (\vx_{i,m}^j)^\top \mB^*\right) \vw_m^* \right) -\mB^\top\mB^* \vw_m^*\right)
\end{align}

Intuitively, assuming $L_m$ is large enough,
\begin{align*}
\frac{1}{L_m} \sum_{i=1}^n \sum_{j=1}^{L_{i,m}}\vx_{i,m}^j (\vx_{i,m}^j)^\top \approx I_d.
\end{align*}
Hence, 
\begin{align*}
  G_m^{-1} \left(\frac{1}{L_m} \sum_{i=1}^n \sum_{j=1}^{L_{i,m}} \left(\mB^\top\vx_{i,m}^j (\vx_{i,m}^j)^\top \mB^*\right) \vw_m^* \right) \approx\mB^\top\mB^* \vw_m^*. 
\end{align*}
This then implies that
\begin{align}
\label{eq:W-update-final-form}
    W^{t+1} = W^*(\mB^*)^\top \mB + F,
\end{align}
where the $m$-th row of $F$ is 
\begin{align*}
    F_m^\top := \left(G_m^{-1} \left(\frac{1}{L_m} \sum_{i=1}^n \sum_{j=1}^{L_{i,m}} \left(\mB^\top\vx_{i,m}^j (\vx_{i,m}^j)^\top \mB^*\right) \vw_m^* \right) -\mB^\top\mB^* \vw_m^*\right)^\top.
\end{align*}
Note the similarity of \eqref{eq:W-update-final-form} to (17) in \cite{collins2021exploiting}. Following a similar analysis as \cite{collins2021exploiting}, we should also be able to bound the 
Frobenius norm of $F$ in terms of $\mathrm{dist}(\mB,\mB^*)$. 

Below, we formalize the argument. First, we have the following lemma.
\begin{lemma}[Update for $W^{t+1}$]
\label{lemma:update_for_W}
For each time $t$, let $L_m^t := \sum_{i=1}^n L_{i,m}^t$ denote the number of samples from domain $m \in [M]$ across the $n$ clients at time $t$. For convenience, we drop the time index unless absolutely necessary. We define the terms
\begin{align*}
    X_m := \frac{1}{L_m}\sum_{i=1}^n \sum_{j=1}^{L_{i,m}} \vx_{i,m}^j (\vx_{i,m}^j)^\top, \quad G_m := \frac{1}{L_m}\sum_{i=1}^n \sum_{j=1}^{L_{i,m}} \left(\mB^\top\vx_{i,m}^j (\vx_{i,m}^j)^\top \mB\right).
\end{align*}
Then, assuming that $G_m$ is invertible, the update for $W$ takes the form
\begin{align}
    W^{t+1} = W^*(\mB^*)^\top \mB + F,
\end{align}
where the $m$-th row of $F$ is
\begin{align}
    F_m^\top := \left(G_m^{-1} \left(\mB^\top X_m (I - \mB\mB^\top)\mB^*\right) \vw_m^*\right)^\top.
\end{align}
\end{lemma}
\begin{proof}
We can express $\nabla_{\vw_m} \sum_{i=1}^n f_{i,m}(\vw_m,\mB)$ as
\begin{align*}
    \nabla_{\vw_m} \sum_{i=1}^n f_{i,m}(\vw_m,\mB) = \sum_{i=1}^n \sum_{j=1}^{L_{i,m}} (\vw_m^\top\mB^\top\vx_{i,m}^j - y_{i,m}^j)\mB^\top\vx_{i,m}^j.
\end{align*}
Since 
$$y_{i,m}^j = (\vw_m^*)^\top (\mB^*)^\top\vx_{i,m}^j,$$
it follows that following Algorithm \ref{algorithm:our-linear},
\begin{align}
% \label{eq:w_m-update-equation-1}
    \underbrace{\left(\frac{1}{L_m}\sum_{i=1}^n \sum_{j=1}^{L_{i,m}} \left(\mB^\top\vx_{i,m}^j (\vx_{i,m}^j)^\top \mB\right)\right)}_{G_m} \vw_m^{t+1} = \frac{1}{L_m} \sum_{i=1}^n \sum_{j=1}^{L_{i,m}} \left(\mB^\top\vx_{i,m}^j (\vx_{i,m}^j)^\top \mB^*\right) \vw_m^*.
\end{align}
Reexpressing, assuming $G_m$ is invertible, we have
\begin{align}
% \label{eq:w_m-update-equation-1}
    \vw_m^{t+1} =\mB^\top\mB^* \vw_m^* + \left(G_m^{-1} \left(\frac{1}{L_m} \sum_{i=1}^n \sum_{j=1}^{L_{i,m}} \left(\mB^\top\vx_{i,m}^j (\vx_{i,m}^j)^\top \mB^*\right) \vw_m^* \right) -\mB^\top\mB^* \vw_m^*\right).
\end{align}
This then implies that
\begin{align}
% \label{eq:W-update-final-form}
    W^{t+1} = W^*(\mB^*)^\top \mB + F,
\end{align}
where the $m$-th row of $F$ is 
\begin{align*}
    F_m^\top &:= \left(G_m^{-1} \left(\frac{1}{L_m} \sum_{i=1}^n \sum_{j=1}^{L_{i,m}} \left(\mB^\top\vx_{i,m}^j (\vx_{i,m}^j)^\top \mB^*\right) \vw_m^* \right) -\mB^\top\mB^* \vw_m^*\right)^\top \\
    &= \left( G_m^{-1} \mB^\top X_m\mB^* \vw_m^* - G_m^{-1} G_m\mB^\top\mB^*\vw_m^* \right)^\top  \\
    &= \left(G_m^{-1} \mB^\top X_m\mB^* \vw_m^* - G_m^{-1} \mB^\top X_m B\mB^\top\mB^*\vw_m^* \right)^\top \\
    &= \left(G_m^{-1} \left(\mB^\top X_m (I - \mB\mB^\top)\mB^*\right) \vw_m^*\right)^\top.
\end{align*}
\end{proof}

% \subsubsection{Bounding $\norm{F}_F$}
\subsubsection{Bounding the Frobenius norm}
We will proceed to bound the Frobenius norm of $F$. We begin by showing that $G_m^{-1}$ exists and (both lower and upper) bounding its spectral norm. 
\begin{lemma}
\label{lemma:bounding_G_m}
Let $L_{\min} := \min_{m \in [M]} L_m$. Let $\delta_k :=  \frac{10 C k \sqrt{\log(M)}}{\sqrt{L_{\min}}}$ for some absolute constant $C$. Suppose that $0 \leq \delta_k < 1$. Then, with probability at least $1 - e^{99k^2 \log(M)}$, $G_m^{-1}$ exists for each $m \in [M]$, and
\begin{align*}
   \norm{G_m^{-1}}_2 \leq \frac{1}{1 - \delta_k} \quad \forall m \in [M].
\end{align*}
\end{lemma}
\begin{proof}
Note that 
\begin{align*}
    G_m := \frac{1}{L_m}\sum_{i=1}^n \sum_{j=1}^{L_{i,m}} \left(\mB^\top\vx_{i,m}^j (\vx_{i,m}^j)^\top \mB\right).
\end{align*}
Let $v_{i,m}^j :=\mB^\top\vx_{i,m}^j$. Since $B^\top B = I$, it follows that each $v_{i,m}^j$ is i.i.d 1-subgaussian. Then, applying the same argument in Theorem 4.6.1 of Vershynin 2018, we have (cf.  equation (4.22) in Vershynin 2018)
\begin{align}
 \sigma_{\min}(G_m) \geq 1 - \underbrace{C \left(\frac{\sqrt{k}}{\sqrt{L_m}} + \frac{z}{\sqrt{L_m}} \right)}_{\delta_{k,m}}
\end{align}
with probability at least $1 - e^{-z^2}$ for $z \geq 0$ and some absolute constant $C$, assuming that $0 \leq \delta_{k,m} \leq 1$. Consider the choice $z = 10 k \sqrt{\log (M)}$. Then, 
\begin{align*}
    \delta_{k,m} = C\left(\frac{\sqrt{k}}{\sqrt{L_m}} + \frac{10 k \log (M)}{\sqrt{L_m}} \right) \leq 10C \frac{k \log M}{\sqrt{L_m}} \leq 10C \frac{k \sqrt{\log M}}{\sqrt{L_{\min}}}.
\end{align*}
Suppose we choose $L_{\min} \geq 1$ such that $\delta_{k,m} < 1$. Then, taking a union bound, with probability at least $1 - m e^{-z^2} = 1 - m \exp(-100 k^2 \log (M)) \geq 1 - \exp(-99k^2 \log(M))$, 
\begin{align}
    \sigma_{\min}(G_m) \geq 1 - \delta_{k,m} \geq 1 - \frac{10C k \sqrt{\log M}}{\sqrt{L_{\min}}} > 0 \quad \forall m \in [M].
\end{align}
Therefore, with probability at least $1 - \exp(-99k^2 \log(M))$, $G_m^{-1}$ exists for every $m \in [M]$, and in addition,
\begin{align*}
\norm{G_m^{-1}}_2 \leq \frac{1}{1 - \delta_k} \quad \forall m \in [M].
\end{align*}
\end{proof}

We next bound the operator norm of term $\mB^\top X_m(I - \mB\mB^\top)\mB^*$. 

\begin{lemma}
\label{lemma:bound_B_X_dist(B,B^*)_norm}
Let $L_{\min} := \min_{m \in [M]} L_m$. Let $\delta_k := \frac{10 C k \sqrt{\log M}}{\sqrt{L_{\min}}}$ for some absolute constant $C$. Suppose $L_{\min}$ is such that $0 \leq \delta_k < 1$. Then, with probability at least $1 - e^{-99k^2 \log M}$,
\begin{align*}
    \norm{\mB^\top X_m(I - \mB\mB^\top)\mB^*}_2 \leq  \mathrm{dist}(\mB^*, B) \delta_k.
\end{align*}
\end{lemma}
\begin{proof}
We will use an $\ep$-net argument, similar to the proof of Theorem 4.6.1 in \cite{vershynin2018high}.

First, by Corollary 4.2.13 in \cite{vershynin2018high}, there exists an $1/4$-net $\mathcal{N}$ of the unit sphere $S^{k-1}$ with cardinality $\mathcal{N} \leq 9^k$. Using Lemma 4.4.1 in \cite{vershynin2018high}, we have that 
\begin{align*}
    \norm{\mB^\top X_m(I - \mB\mB^\top)\mB^*}_2 \leq 2 \max_{z \in \mathcal{N}} \abs*{\brac*{\left(\mB^\top X_m(I - \mB\mB^\top)\mB^*\right)z,z}}.
\end{align*}
To prove our result, by applying a union bound over $m \in [M]$, it suffices to show that with the probability at least $1 - e^{-100k^2 \log M}$, 
\begin{align*}
    \max_{z \in \mathcal{N}} \abs*{\brac*{\left(\mB^\top X_m(I - \mB\mB^\top)\mB^*\right)z,z}} \leq \frac{\delta_{k_m}}{2} \quad \forall m \in [M],
\end{align*}
where we recall that 
$$ \delta_{k,m} = C\left(\frac{\sqrt{k}}{\sqrt{L_m}} + \frac{10 k \log (M)}{\sqrt{L_m}} \right) \leq \delta_k.$$
We will assume that $\min_m L_m := L_{\min} \geq 1$ is chosen large enough such that $\delta_{k,m} \leq 1.$

For a fixed $z \in S^{k-1}$, observe that 
\begin{align*}
 \brac*{\left(\mB^\top X_m(I - \mB\mB^\top)\mB^*\right)z,z} &= \frac{1}{L_m} \sum_{i=1}^n \sum_{j=1}^{L_{i,m}} \brac*{\left(B^\top\vx_{i,m}^j (\vx_{i,m}^j)^\top (I-\mB\mB^\top)\mB^*\right)z,z} \\
 &:= \frac{1}{L_m} \sum_{i=1}^n \sum_{j=1}^{L_{i,m}} (z^\top u_{i,m}^j)((v_{i,m}^j)^\top z),
\end{align*}
where we defined $u_{i,m}^j :=\mB^\top\vx_{i,m}^j$, and $v_{i,m}^j = (\mB^*)^\top (I - \mB\mB^\top)\vx_{i,m}^j$. 

Since each $x_{i,m}^j$ is 1-subgaussian, $\norm{\mB}_2 = 1$, and $\norm{(I-\mB\mB^\top)\mB^*}_2 = \mathrm{dist}(\mB^*,\mB)$, it follows that $z^\top u_{i,m}^j$ is subgaussian with norm at most 1, and $(v_{i,m}^j)^\top z$ is subgaussian with norm at most $\mathrm{dist}(\mB^*,\mB)$. Thus, the random variable $\alpha_{i,m}^j := (z^\top u_{i,m}^j)((v_{i,m}^j)^\top z)$ (for a fixed unit $z$) is sub-exponential with sub-exponential norm at most $\mathrm{dist}(\mB^*,\mB)$. Moreover, note that $\alpha_{i,m}^j$ is mean-zero, since
\begin{align*}
    \bbE[u_{i,m}^j (v_{i,m}^j)^\top] &= \bbE[B^\top\vx_{i,m}^j (\vx_{i,m}^j)^\top (I-\mB\mB^\top)\mB^*] \\
    &=\mB^\top(I-\mB\mB^\top)\mB^* = 0,
\end{align*}
as $x_{i,m}^j$ is assumed to have identity covariance. Thus, the $\alpha_{i,m}^j$'s are i.i.d mean-zero subexponential variables each with subexponential norm at most $\mathrm{dist}(\mB^*,\mB)$.  Hence, by Bernstein's inequality (cf. Corollary 2.8.3 in \cite{vershynin2018high}),
\begin{align*}
    &\bbP \left(\abs*{\brac*{\left(\mB^\top X_m(I - \mB\mB^\top)\mB^*\right)z,z}} \geq \frac{\delta_{k,m}\mathrm{dist}(\mB^*,\mB) }{2} \right)  \\
    &=
    \bbP\left(\abs*{\frac{1}{L_m} \sum_{i=1}^n \sum_{j=1}^{L_{i,m}} \alpha_{i,m}^j} \geq \frac{\delta_{k,m}\mathrm{dist}(\mB^*,\mB) }{2} \right) \\
    &\leq 2 \exp\left(-c \min(\frac{\delta_{k,m} \mathrm{dist}(\mB^*,\mB)}{\mathrm{dist}(\mB^*,\mB)},\left(\frac{\delta_{k,m} \mathrm{dist}(\mB^*,\mB)}{\mathrm{dist}(\mB^*,\mB)}\right)^2) L_m \right) \\
    &= 2\exp(-c \delta_{k,m}^2 L_m) \\
    &\leq 2\exp(-c C^2\left(k +  100 k^2 \log (M) \right)).
\end{align*}
Above we used the assumption that $\delta_{k,m} \leq 1$ to simplify the minimum operator in the exponent.

Taking a union bound over each $z \in \mathcal{N}$, it follows that 
\begin{align*}
    \bbP\left(\norm{\mB^\top X_m(I - \mB\mB^\top)\mB^*}_2 \geq \delta_{k,m}  \mathrm{dist}(\mB^*,\mB) \right) &\leq
    \bbP\left(2\max_{z \in \mathcal{N}}\abs*{\brac*{\left(\mB^\top X_m(I - \mB\mB^\top)\mB^*\right)z,z}} \geq \delta_{k,m} \mathrm{dist}(\mB^*,\mB)\right) \\
    &\leq 2 \cdot  9^k \exp(-c C^2\left(k +  100 k^2 \log (M) \right)) \\
    &\leq \exp(-100k^2 \log M),
\end{align*}
where the last inequality follows by picking $C$ large enough (but still it is an absolute constant). By applying a union bound over the domains $m \in [M]$, this then completes our proof. 
\end{proof}

We are now finally ready to bound $\norm{F}_F$.
\begin{lemma}
\label{lemma:bound_F_frob_norm}
Let $L_{\min} := \min_{m \in [M]} L_m$. Let $\delta_k :=  \frac{10 C k \sqrt{\log(M)}}{\sqrt{L_{\min}}}$ for some absolute constant $C$. Suppose that $0 \leq \delta_k < 1$. Then, with probability at least $1 - 2e^{-99k^2 \log(M)}$,
\begin{align*}
    \norm{F}_F \leq \frac{\delta_k}{1 - \delta_k} \mathrm{dist}(\mB^*,\mB) \norm{W^*}_F.
\end{align*}
\end{lemma}
\begin{proof}
By Lemma \ref{lemma:bounding_G_m} and Lemma \ref{lemma:bound_B_X_dist(B,B^*)_norm}, we have that with probability at least $1 - 2e^{99k^2 \log M}$, 
\begin{align*}
    \norm*{G_m^{-1}(\mB^\top X_m(I - \mB\mB^\top)\mB^*)}_2 &\leq \norm*{G_m^{-1}}_2\norm*{\mB^\top X_m(I - \mB\mB^\top)\mB^*}_2 \\
    &\leq \frac{1}{1 - \delta_k} \delta_k \mathrm{dist}(\mB^*,\mB).
\end{align*}
The proof then follows by recalling that the $m$-th row, $F_m^\top$, takes the form
\begin{align*}
    F_m^\top = \left(G_m^{-1} \left(\mB^\top X_m (I - \mB\mB^\top)\mB^*\right) \vw_m^*\right)^\top.
\end{align*}
\end{proof}

\subsubsection{Analysis of updating the embedding weights}
% \subsubsection{Analysis of update for $B^{t+1}$}
Similarly to \cite{collins2021exploiting}, we define
\begin{align*}
    Q^t = W^{t+1}(\mB^t)^\top - (W^*)(\mB^*)^\top.
\end{align*}
Below, we drop the time index and use $\mB,Q,W$ to denote $\mB^t,Q^t,$ and $W^{t+1}$ respectively.
Based on algorithm \ref{algorithm:our-linear}, we have that 
\begin{align}
    \tilde{\mB}^{t+1} &= \mB - \frac{\eta}{M}\sum_{m=1}^M \frac{1}{L_m} \sum_{i=1}^n \sum_{j=1}^{L_{i,m}} (\vw_m^\top\mB^\top\vx_{i,m}^j - y_{i,m}^j )\vx_{i,m}^j \vw_m^\top \nonumber \\
    &= \mB - \frac{\eta}{M} \sum_{m=1}^M \frac{1}{L_m} \sum_{i=1}^n \sum_{j=1}^{L_{i,m}} \left(\brac*{A_{i,m}^j, W\mB^\top} - \brac*{A_{i,m}^j, W^*(\mB^*)^\top} \right) (A_{i,m}^j)^\top W, \quad \quad A_{i,m}^j := e_m (\vx_{i,m}^j)^\top \nonumber \\
    &= \mB - \frac{\eta}{M} \sum_{m=1}^M \frac{1}{L_m} \sum_{i=1}^n \sum_{j=1}^{L_{i,m}} \left(\brac*{A_{i,m}^j, Q} \right) (A_{i,m}^j)^\top W \nonumber \\
    &= \mB - \frac{\eta}{M} \sum_{m=1}^M \frac{1}{L_m} \sum_{i=1}^n \sum_{j=1}^{L_{i,m}}\vx_{i,m}^j (\vx_{i,m}^j)^\top q_m \vw_m^\top \nonumber\\
    &= \mB - \frac{\eta}{M} Q^\top W - \underbrace{\left[\frac{\eta}{M} \sum_{m=1}^M \frac{1}{L_m} \sum_{i=1}^n \sum_{j=1}^{L_{i,m}}\vx_{i,m}^j (\vx_{i,m}^j)^\top q_m \vw_m^\top -  \frac{\eta}{M} Q^\top W \right]}_{H_Q}. \label{eq:B-update-eqn-1}
\end{align}
Above, we define $q_m \in \bbR^d$ to denote the $m$-th row of $Q$ (viewed as a column vector). Note again that since 
\begin{align*}
    \frac{1}{L_m} \sum_{i=1}^n \sum_{j=1}^{L_{i,m}}\vx_{i,m}^j (\vx_{i,m}^j)^\top \approx I_d,
\end{align*}
the term $H_Q$ in \eqref{eq:B-update-eqn-1} can be appropriately bounded. Note the resemblance of \eqref{eq:B-update-eqn-1} to (53) in \cite{collins2021exploiting}; the crucial difference is that we will need to lower bound $\frac{1}{m} \sigma_{\min}^2(W^*)$, instead of $\frac{1}{n} \sigma_{\min}^2(W^*)$ as in \cite{collins2021exploiting}. Thus we should be able to carry out the rest of the analysis in a similar way to the outline in \cite{collins2021exploiting} and derive an analogous result to Theorem 1 in \cite{collins2021exploiting}.  

We first bound the error term $H_Q$.
\begin{lemma}
\label{lemma:bounding_H_Q}
Let 
\begin{align*}
    H_Q^t := \frac{\eta}{M} \sum_{m=1}^M \frac{1}{L_m} \sum_{i=1}^n \sum_{j=1}^{L_{i,m}}\vx_{i,m}^j (\vx_{i,m}^j)^\top q_m (\vw_m^{t+1})^\top -  \frac{\eta}{M} (Q^t)^\top W^{t+1}.
\end{align*}
Let $\gamma_k :=  \frac{20  k \sqrt{d}}{c\sqrt{nL}}$ for some absolute constant $c$. Suppose that $0 \leq \gamma_k < k$.
Then, for any $t$, with probability at least $1 - \exp(-90d) - 2e^{-99k^2 \log M}$, 
\begin{align*}
    \norm{H_Q^t}_2 \leq \eta \gamma_k \mathrm{dist}(\mB^*,\mB^t).
\end{align*}
\end{lemma}
\begin{proof}
As before, we may omit the time superscript $t$ in cases where it is clear for notational convenience.
The proof is based on the argument in Lemma 5 in \cite{collins2021exploiting}. Again, the main tool is an $\ep$-net argument. We first bound $\norm{q_m}_2$ and $\norm{\vw_m}_2$.

\textbf{Bounding $q_m$:} With probability at least $1 - 2e^{-99k^2 \log M}$, for each $m \in [M]$, we have that
\begin{align*}
    \norm{q_m}_2 &= \norm*{\mB^t ((\mB^t)^\top \mB^* \vw_m^* + F_m) - \mB^* \vw_m^*}_2 \\
    &\leq \norm*{(\mB^t (\mB^t)^\top - I) \mB^* \vw_m^*}_2 + \norm*{\mB^t F_m}_2\\
    &\leq \mathrm{dist}(\mB^t, \mB^*) \norm{\vw_m^*}_2 + \norm{F_m}_2 \\
    &\leq \sqrt{k} \mathrm{dist}(\mB^t, \mB^*) + \frac{\delta_k}{1 - \delta_k} \mathrm{dist}(\mB^t, \mB^*) \norm{\vw_m^*}_2 \\
    &\leq 2\sqrt{k} \mathrm{dist}(\mB^t, \mB^*).
\end{align*}
Above, we utilized the assumption that $\norm{\vw_m^*}_2 \leq \sqrt{k}$, the orthonormality of $\mB^t$ (which was derived as the orthogonal matrix from a Gram-Schmidt procedure), the assumption that $0 < \delta_k \leq 1/2$, as well Lemma \ref{lemma:bound_F_frob_norm} which bounds $\norm{F_m}_2$ with high probability. 

\textbf{Bounding $\vw_m$:} Note that for notational convenience, we let $\vw_m$ denote $\vw_m^{t+1}$. For each $t$ and every $m \in [M]$, we have that
\begin{align*}
    \norm{\vw_m^{t+1}}_2 &= \norm*{(\mB^t)^\top \mB^* \vw_m^* + F_m}_2 \\
    &\leq \norm*{\vw_m^*}_2 + \norm*{F_m}_2 \\
    &\leq \norm*{\vw_m^*}_2 + \frac{\delta_k}{1 - \delta_k} \mathrm{dist}(\mB^t, \mB^*) \norm{\vw_m^*}_2  \\
    &\leq 3\sqrt{k},
\end{align*}
with probability at least $1 - 2e^{-99k^2 \log M}$, where again we used Lemma \ref{lemma:bound_F_frob_norm} to handle $\norm{F_m}_2$, the assumption that $\delta_k < 1/2$, and the fact that $\mathrm{dist}(\mB^t, \mB^*) \leq 2$.

For the rest of the proof, we condition on the event 
\begin{align*}
    \mathcal{E} := \left\{\norm*{q_m}_2 \leq 2 \sqrt{k} \mathrm{dist}(\mB^t,\mB^*) \mbox{ and } \norm*{\vw_m}_2 \leq 3\sqrt{k} \quad \forall m \in [M]\right\},
\end{align*}
which holds with probability at least $1 - 2e^{-99k^2 \log M}$.

\textbf{$\ep$-net argument to bound $H_Q$:} Again, note that there exists an $1/4$-net $\mathcal{N}_k$ of the unit sphere $S^{k-1}$ and an $1/4$-net $\mathcal{N}_d$ of the unit sphere $S^{d-1}$ with cardinalities less than or equal to $9^k$ and $9^d$ respectively. 

Note now that by Equation 4.13 in \cite{vershynin2018high}, we have
\begin{align}
    \norm*{H_Q}_2 &= \norm*{\frac{\eta}{M} \sum_{m=1}^M \frac{1}{L_m} \sum_{i=1}^n \sum_{j=1}^{L_{i,m}}\vx_{i,m}^j (\vx_{i,m}^j)^\top q_m \vw_m^\top -  \frac{\eta}{M} Q^\top W}_2 \nonumber \\
    &\leq 2\eta  \max_{u \in \mathcal{N}_d, v \in \mathcal{N}_k} \frac{1}{M} \sum_{m=1}^M \frac{1}{L_m} \sum_{i=1}^n \sum_{j=1}^{L_{i,m}} \brac*{ \left(\vx_{i,m}^j (\vx_{i,m}^j)^\top q_m \vw_m^\top -  q_m \vw_m^\top \right) u, v} \nonumber \\
    &= 2\eta \max_{u \in \mathcal{N}_d, v \in \mathcal{N}_k} \frac{1}{M} \sum_{m=1}^M \frac{1}{L_m} \sum_{i=1}^n \sum_{j=1}^{L_{i,m}} \left[\left(u^\top\vx_{i,m}^j\right) \left( (\vx_{i,m}^j)^\top q_m \vw_m^\top v \right) - \brac*{q_m \vw_m^\top u, v}\right] \label{eq:H_Q_expression_first_simplification}
\end{align}
Fix now a $u \in \mathcal{N}_d$ and $v \in \mathcal{N}_k$. Note now that $\left(u^\top\vx_{i,m}^j\right) \left( (\vx_{i,m}^j)^\top q_m \vw_m^\top v \right)$ is subexponential with norm less than or equal to $\norm{q_m}_2 \norm{\vw_m}_2 \leq 6k \mathrm{dist}(\mB^t, \mB^*)$, since it is the product of two subgaussian variables $u^\top\vx_{i,m}^j$ and $(\vx_{i,m}^j)^\top q_m \vw_m^\top v $ with subgaussian norms bounded by 1 and $\norm{q_m}_2 \norm{\vw_m}_2$ respectively. Note also that 
\begin{align*}
    \mathbb{E} \left[\left(u^\top\vx_{i,m}^j\right) \left( (\vx_{i,m}^j)^\top q_m \vw_m^\top v \right) \right] = \bbE \left[  \brac*{q_m \vw_m^\top u, v} \right].
\end{align*}
Thus, by Bernstein's inequality, carrying on from \eqref{eq:H_Q_expression_first_simplification}, we have that 
\begin{align*}
    & \quad \quad \bbP \left(\frac{1}{M} \sum_{m=1}^M \frac{1}{L_m} \sum_{i=1}^n \sum_{j=1}^{L_{i,m}} \left[\left(u^\top\vx_{i,m}^j\right) \left( (\vx_{i,m}^j)^\top q_m \vw_m^\top v \right) - \brac*{q_m \vw_m^\top u, v}\right] \geq  \rho \right) \\
    &\leq \exp \left(-c n L \min \left(\frac{\rho}{6k \mathrm{dist}(\mB^t,\mB^*)}, \left(\frac{\rho}{k \mathrm{dist}(\mB^t,\mB^*)} \right)^2 \right)  \right) \\
    &\leq \exp \left(-c n L  \left(\frac{\rho}{k \mathrm{dist}(\mB^t,\mB^*)} \right)^2 \right),
\end{align*}
where we will choose $\rho$ such that $\frac{\rho}{k \mathrm{dist}(\mB^t,\mB^*)} \leq 1$ to simplify the exponent in the way we did, and $c$ is an absolute constant that may change from line to line. Above, we also used the fact that $\sum_{m=1}^M L_m = nL$ (recall that $L$ is the total number of samples per agent and there are $n$ agents).

Consider the choice 
\begin{align*}
    \rho = 10 \frac{k \sqrt{d} \mathrm{dist}(\mB^t, \mB^*)}{c\sqrt{nL}}.
\end{align*}
Then, 
\begin{align*}
    & \quad \quad \bbP \left(\frac{1}{M} \sum_{m=1}^M \frac{1}{L_m} \sum_{i=1}^n \sum_{j=1}^{L_{i,m}} \left[\left(u^\top\vx_{i,m}^j\right) \left( (\vx_{i,m}^j)^\top q_m \vw_m^\top v \right) - \brac*{q_m \vw_m^\top u, v}\right] \geq  \rho \right) \\
    &\leq \exp \left(-c n L  \left(\frac{\rho}{k \mathrm{dist}(\mB^t,\mB^*)} \right)^2 \right) \\
    &\leq \exp(-100 d).
\end{align*}
Taking a union bound over all $u \in \mathcal{N}_d$ and $v \in \mathcal{N}_k$, it follows then that 
\begin{align*}
    \bbP \left(\frac{\norm*{H_Q}_2}{\eta} \geq  2 \rho \right) \leq 9^{d+k} \exp(-100 d) \leq \exp(-90d),
\end{align*}
where above we used the fact that $d \geq k$. 
\end{proof}

\subsubsection{Combining earlier argument: convergence of \our}

As seen in Lemma \ref{lemma:bound_F_frob_norm}, we require that $L_{\min} := \min_{m \in [M]} L_m$ to be lower bounded. However, since $L_m$ is a stochastic variable, we are unable to directly lower bound it. Below, we provide a result that converts a lower bound on each client's sample size $L$ (a deterministic quantity we can control) to a high-probability lower bound on $L_{\min}$.
\begin{lemma}
\label{lemma:lbd_L_to_lbd_L_min}
Let $L_{\min} := \min_{m \in [M]} L_m$. For any $\alpha > 0$, suppose that for each $m \in [M]$,
\begin{align*}
    L \geq \max\left\{\frac{182 \log M}{\sum_{i=1}^n \pi_{i,m}}, \frac{16}{\sum_{i=1}^n \pi_{i,m}}, \frac{2\alpha}{\sum_{i=1}^n \pi_{i,m}}\right\}.
\end{align*}
Then, with probability at least $1-\exp(-90)$, 
\begin{align*}
    L_{\min} \geq \alpha.
\end{align*}
\end{lemma}
\begin{proof}
Note that 
\begin{align*}
    L_m = \sum_{i=1}^n \sum_{j=1}^L \ind(\mathrm{domain}(x_i^j) = m),
\end{align*}
which is a sum of $nL$ independent random variables bounded between 0 and 1. Moreover, 
$$\bbE[L_m] = \sum_{i=1}^n \pi_{i,m} L,$$
where $\pi_{i,m}$ is the probability that a datapoint comes from domain $m$ for client $i$. Note finally that $$\bbE[\left(\ind(\mathrm{domain}(x_i^j) = m)\right)^2] = \pi_{i,m}.$$ Hence, by Bernstein's inequality, it follows that for any $s > 0$, 
\begin{align*}
    \bbP\left(L_{i,m} \leq \sum_{i=1}^n \pi_{i,m} L - s\right) \leq \exp\left(-\frac{s^2/2}{\sum_{i=1}^n \sum_{j=1}^L \pi_{i,m} + s/3} \right).
\end{align*}
Since we wish to perform union bound over the $M$ domains, we seek to choose $s$ and $L$ such  that 
\begin{align*}
    \exp\left(-\frac{s^2/2}{\sum_{i=1}^n \sum_{j=1}^L \pi_{i,m} + s/3} \right) \leq \exp\left(-91 \log M \right), 
\end{align*}
so that
\begin{align*}
    M \exp\left(-\frac{s^2/2}{\sum_{i=1}^n \sum_{j=1}^L \pi_{i,m} + s/3} \right) \leq M \exp\left(-91 \log M \right) \leq \exp\left( -90\log M\right). 
\end{align*}
To this end, note that we need
\begin{align*}
    & \quad \quad \frac{s^2/2}{\sum_{i=1}^n \sum_{j=1}^L \pi_{i,m} + s/3} \geq 91 \log M \\
    &\iff s^2 \geq 2\cdot  91 \log M \left(\sum_{i=1}^n \sum_{j=1}^L \pi_{i,m} + s/3 \right) \\
    &\iff s \geq \sqrt{182 \log M \sum_{i=1}^n \pi_{i,m}L + \left(\frac{182 \log M}{3} \right)^2} + \frac{182 \log M}{3}  
\end{align*}
Suppose we pick $L$ such that
$$\sum_{i=1}^n \pi_{i,m} L \geq 182 \log M,$$
so that 
$$\sqrt{182 \log M \sum_{i=1}^n \pi_{i,m}L + \left(\frac{182 \log M}{3} \right)^2} + \frac{182 \log M}{3} \leq 2 \sqrt{\sum_{i=1}^n \pi_{i,m}L}.$$
Then, by picking $s = 2\sqrt{\sum_{i=1}^n \pi_{i,m}L}$, it follows that 
\begin{align*}
\exp\left(-\frac{s^2/2}{\sum_{i=1}^n \sum_{j=1}^L \pi_{i,m} + s/3} \right) \leq \exp\left(-91 \log M \right), 
\end{align*}
such that for each $m \in [M]$,
\begin{align*}
\bbP\left(L_{i,m} \leq \sum_{i=1}^n \pi_{i,m} L - 2\sqrt{\sum_{i=1}^n \pi_{i,m}L}\right) \leq \exp(-91 \log M).
\end{align*}
By choosing $L$ such that 
\begin{align*}
\sqrt{\sum_{i=1}^n \pi_{i,m}L} \geq 4,
\end{align*}
it follows that 
\begin{align*}
\bbP\left(L_{i,m} \leq \frac{\sum_{i=1}^n \pi_{i,m} L}{2} \right) \leq \exp(-91 \log M).
\end{align*}
The result now follows by choosing $L$ such that it also satisfies
\begin{align*}
    \frac{\sum_{i=1}^n \pi_{i,m} L}{2} \geq \alpha
\end{align*}
for each $m$.

\end{proof}

\begin{lemma}[Descent lemma]
\label{lemma:fed_mdr_descent}
Define $E_0 := 1 - \mathrm{dist}^2(\mB^0,\mB^*)$ and $\bar{\sigma}_{\max,*} := \sigma_{\max} \left(\frac{1}{\sqrt{M}} W^* \right)$ and $\bar{\sigma}_{\min,*} := \sigma_{\min}\left(\frac{1}{\sqrt{M}} W^* \right)$. 
Let $\kappa := \frac{\bar{\sigma}_{\max,*}}{\bar{\sigma}_{\min,*}}$. Consider any iteration $t$.

Suppose that 
\begin{align*}
    L \geq \left(\frac{400dk^2}{n c}\right) \frac{1}{\left(\min \left\{\frac{1}{2}, 8 E_0/(25 \cdot 5 \kappa^2)\right\}\right)^2},
\end{align*}
where $c > 0$ is absolute constant.
Suppose also that
\begin{align*}
L \geq \max\left\{\frac{182 \log M}{\sum_{i=1}^n \pi_{i,m}}, \frac{16}{\sum_{i=1}^n \pi_{i,m}}, \frac{2\left(100 C k^2 \log M \right)  \frac{1}{\left(\min \left\{\frac{1}{2}, 8 E_0/(25 \cdot 5 \kappa^2)\right\}\right)^2}}{\sum_{i=1}^n \pi_{i,m}}\right\},
\end{align*}
which by Lemma \ref{lemma:lbd_L_to_lbd_L_min}, ensures that with probability at least $1 - e^{-90}$,
\begin{align*}
    L_{\min}^t \geq \left(100 C k^2 \log M \right)  \frac{1}{\left(\min \left\{\frac{1}{2}, 8 E_0/(25 \cdot 5 \kappa^2)\right\}\right)^2},
\end{align*}
where $L_{\min}^t = \min_{m \in [M]} L_m^t$ and $C > 0$ is an absolute constant. 

Then, for any $\eta \leq 1/(4\bar{\sigma}_{\max,*}^2)$, we have
\begin{align*}
    \mathrm{dist}(\mB^{t+1},\mB^*) \leq (1 - \eta E_0 \bar{\sigma}_{\min,*}/2)^{1/2} \mathrm{dist}(\mB^t, \mB^*),
\end{align*}
with probability at least $1 - e^{-80}$.
\end{lemma}
\begin{proof}
We begin with the observation that
\begin{align*}
    &W^{t+1} = W^* (\mB^*)^\top \mB^t + F^t \\
    &\bar{\mB}^{t+1} = \mB^t - \frac{\eta}{M} (Q^t)^\top W^{t+1} - H_Q^t,
\end{align*}
where 
$$Q^t =  W^{t+1}(\mB^t)^\top - (W^*)(\mB^*)^\top,$$
and 
$$H_Q^t := \frac{\eta}{M} \sum_{m=1}^M \frac{1}{L_m} \sum_{i=1}^n \sum_{j=1}^{L_{i,m}} \vx_{i,m}^j (\vx_{i,m}^j)^\top q_m (\vw_m^{t+1})^\top -  \frac{\eta}{M} (Q^t)^\top W^{t+1}.$$
Above $\bar{\mB}^{t+1}$ denotes the estimate of $\mB$ before we perform the $QR$ decomposition.
We note that the updates for $W$ and $\mB$ are exactly analogous to the updates for $W$ and $\mB$ as seen in the proof of Lemma 6 in \cite{collins2021exploiting}. The only two differences are
\begin{enumerate}
    \item The definitions of $F$ in our paper and \cite{collins2021exploiting} are slightly different. However, in both cases, 
    \begin{align*}
        \norm{F}_F \leq \frac{\delta_k}{1 - \delta_k} \mathrm{dist}(\mB^*,\mB) \norm{W^*}_F
    \end{align*}
    for some term $\delta_k \leq 1/2$ with high probabilities. In our case, this event holds with probability at least $1 - 2\exp(-99k^2 \log M)$, whilst in \cite{collins2021exploiting}, the event holds with probability at least $1 - \exp(-110k^2 \log n)$.
    \item The update for $\bar{\mB}^{t+1}$ in \cite{collins2021exploiting} takes the form
    \begin{align*}
        \bar{\mB}^{t+1} = \mB^t - \frac{\eta}{rn} (Q^t)^\top W^{t+1} - \frac{\eta}{rn} \left(\frac{1}{m} \mathcal{A}^\dag \mathcal{A}(Q^t) - Q^t \right)^\top W^{t+1},
    \end{align*}
    where $0 \leq r \leq 1$ is a ratio term used in \cite{collins2021exploiting}, and $m$ above represents the number of samples used by each learner in \cite{collins2021exploiting} (which is different from our use of $m$ as an index over the domains). However, we note that with high probabilities,
    \begin{align*}
        &\norm*{H_Q^t}_2 \leq \eta \gamma_k \mathrm{dist}(\mB^t,\mB^*), \\
        &\norm*{\frac{\eta}{rn} \left(\frac{1}{m} \mathcal{A}^\dagger \mathcal{A}(Q^t) - Q^t \right)^\top W^{t+1}}_2 \leq \eta \gamma_k \mathrm{dist}(\mB^t,\mB^*),
    \end{align*}
    where the definition of $\gamma_k$ in both papers differ but both satisfy the assumption that $\gamma_k \leq k.$ 
\end{enumerate}
Due to these similarities in the updates for $W^{t+1}$ and $B^{t+1}$ with the update in \cite{collins2021exploiting}, the proof of this lemma follows naturally from the proof of Lemma 6 in \cite{collins2021exploiting}, by plugging in $\frac{\eta}{M} (Q^t)^\top W^{t+1}$ in the update for $\bar{\mB}^{t+1}$ in place of $\frac{\eta}{rn} (Q^t)^\top W^{t+1}$ as in \cite{collins2021exploiting}. In particular, following the same analysis as in \cite{collins2021exploiting}, we see that on the events in Lemma \ref{lemma:bound_F_frob_norm} and Lemma \ref{lemma:bounding_H_Q}, following the equation immediately after Equation (84) in \cite{collins2021exploiting}, we have
\begin{align*}
    \mathrm{dist}(\mB^t,\mB^*) \leq \frac{1}{\sqrt{1 - 4\eta \frac{\bar{\delta}_k}{(1 - \bar{\delta}_k)^2} \bar{\sigma}_{\max,*^2}}} \left( 1- \eta \bar{\sigma}_{\min,*}^2E_0 + 2\eta \frac{\bar{\delta}_k}{(1 - \bar{\delta}_k)^2} \bar{\sigma}_{\max,*}^2 \right) \mathrm{dist}(\mB^t,\mB^*),
\end{align*}
where in our case $\bar{\delta}_k = \delta_k + \gamma_k$. Then, by choosing
\begin{align}
\bar{\delta}_k < 16E_0/(25 \cdot 5 \kappa^2), \label{eq:bar_delta_k_requirement}
\end{align}
it follows that $\bar{\delta}_k < 1/5$, and so 
\begin{align*}
1- \eta \bar{\sigma}_{\min,*}^2E_0 + 2\eta \frac{\bar{\delta}_k}{(1 - \bar{\delta}_k)^2} \bar{\sigma}_{\max,*}^2 &\leq 1 - 4\eta \frac{\bar{\delta_k}}{(1 - \bar{\delta}_k^2)}\bar{\sigma}_{\max,*}^2 \leq 1 - \eta E_0 \bar{\sigma}_{\min,*}^2/2,
\end{align*}
as in equation (85) in \cite{collins2021exploiting}, such that
\begin{align*}
    \mathrm{dist}(\mB^{t+1},\mB^*) \leq (1 - \eta E_0 \bar{\sigma}_{\min,*}^2/2)^{1/2}\mathrm{dist}(\mB^t,\mB^*).
\end{align*}

It remains for us to understand what the constraint on $\bar{\delta}_k$ spelt out in \eqref{eq:bar_delta_k_requirement}, and the constraints on $\delta_k$ and $\gamma_k$ (in Lemmas \ref{lemma:bound_F_frob_norm} and \ref{lemma:bounding_H_Q} respectively)
mean in our choice of the sample size $L$ for each agent, and the domain size $L_m$ at each iteration. Observe that we need 
\begin{align}
    &\delta_k = \frac{10 C k\sqrt{\log M}}{\sqrt{L_{\min}}}\leq \frac{1}{2}, \label{eq:delta_k_requirement}\\
    &\gamma_k = \frac{20 k \sqrt{d}}{c \sqrt{nL}} \leq \frac{1}{2}, \label{eq:rho_k_requirement}\\
    &\bar{\delta}_k = \delta_k + \gamma_k = \frac{10 C k\sqrt{\log M}}{\sqrt{L_{\min}}} + \frac{20 k \sqrt{d}}{c \sqrt{nL}} \leq 16E_0/(25 \cdot 5 \kappa^2), \label{eq:bar_delta_k_requirement_together}
\end{align}
where $c, C > 0$ are absolute constants. By choosing 
\begin{align*}
    &L_{\min} \geq \left(100 C k^2 \log M \right)  \frac{1}{\left(\min \left\{\frac{1}{2}, 8 E_0/(25 \cdot 5 \kappa^2)\right\}\right)^2} \\
    &L \geq \left(\frac{400dk^2}{n c}\right) \frac{1}{\left(\min \left\{\frac{1}{2}, 8 E_0/(25 \cdot 5 \kappa^2)\right\}\right)^2},
\end{align*}
we ensure that the requirements in \eqref{eq:delta_k_requirement}, \eqref{eq:rho_k_requirement} and \eqref{eq:bar_delta_k_requirement_together} are all satisfied. 

The final result then follows by applying Lemma \ref{lemma:lbd_L_to_lbd_L_min}.
\end{proof}

This then yields the following convergence result, which is a more complete statement of \ref{theorem:fed-MDR-linear}.

\begin{theorem}
[Convergence result for Algorithm \ref{algorithm:our-linear}]

Define $E_0 := 1 - \mathrm{dist}^2(\mB^0,\mB^*)$ and $\bar{\sigma}_{\max,*} := \sigma_{\max} \left(\frac{1}{\sqrt{M}} W^* \right)$ and $\bar{\sigma}_{\min,*} := \sigma_{\min}\left(\frac{1}{\sqrt{M}} W^* \right)$. 
Let $\kappa := \frac{\bar{\sigma}_{\max,*}}{\bar{\sigma}_{\min,*}}$.

Suppose that 
\begin{align*}
    L \geq \left(\frac{400dk^2}{n c}\right) \frac{1}{\min \left\{\frac{1}{2}, 8 E_0/(25 \cdot 5 \kappa^2)\right\}},
\end{align*}
where $c > 0$ is absolute constant.
Suppose also that
\begin{align*}
L \geq \max\left\{\frac{182 \log M}{\sum_{i=1}^n \pi_{i,m}}, \frac{16}{\sum_{i=1}^n \pi_{i,m}}, \frac{2\left(100 C k^2 \log M \right)  \frac{1}{\min\left\{1/2,8 E_0/(25 \cdot 5 \kappa^2) \right\}}}{\sum_{i=1}^n \pi_{i,m}}\right\}.
\end{align*}

Then, for any $\eta \leq 1/(4\bar{\sigma}_{\max,*}^2)$, we have
\begin{align*}
    \mathrm{dist}(\mB^{t+1},\mB^*) \leq (1 - \eta E_0 \bar{\sigma}_{\min,*}/2)^{1/2} \mathrm{dist}(\mB^t, \mB^*),
\end{align*}
with probability at least $1 - e^{-80}$.
 Then for any $T$ and any $\eta \leq 1/(4\sigma_{\max,*}^2)$, we have
\begin{align}
    \mathrm{dist}(\mB^t, \mB^*) \leq (1 - \eta E_0 \bar{\sigma}_{\min,*}^2/2)^{T/2} \mathrm{dist}(\mB^0,\mB^*),
\end{align}
with probability at least $1- T e^{-80}$.
\end{theorem}
By assuming that $\sigma_{\min,*}^2 > 0$, the bound in Theorem 1 decays exponentially. We note that the total number of samples required per client scales with $L \log (1/\ep)$. In addition, in order for the result to be meaningful, we implicitly assume that $E_0$ is close to 1 such that 
$$0 < 1 - \eta E_0 \bar{\sigma}_{\min}^2 < 1.$$
To do so, we note it is possible to choose $\mB^0$ such that $\mathrm{dist}(\mB_0,\mB^*)$ is close enough to 0, with only a logarithmic increase in sample complexity when the number of samples is uniform across the domains. The argument follows the proof of Theorem 3 in \cite{tripuraneni2021provable}.

\begin{theorem}
Suppose Assumptions \ref{assumption:data_mean_zero_subgaussian}, \ref{assumption:domain-diversity}, \ref{assumption:incoherence_orthonormal} all hold. Suppose also that $x_{i}^{0,j} \sim \mathcal{N}(0,I_d)$ independently for all $i \in [n]$. Suppose each client $i$ sends the server $Z_{i} :=  \sum_{j=1}^{L^0} (y_{i}^{0,j})^2 x_{i}^j (x_{i}^j)^\top$, as well as the integer value of $L_{i}$, such that the server can compute $Z := \frac{1}{nL^0} \sum_{i=1}^n Z_{i}$. Then, the server computes $UDU^\top \leftarrow \mbox{rank-k SVD}\left(Z \right)$, and sets $\mB^0 := U$. 
Let 
\begin{align*}
    \bar{\Lambda} = \frac{1}{nL^0} \sum_{i=1}^n \sum_{j=1}^{L^0} w_{m(i,j)}^* (w_{m(i,j)}^*)^\top,
\end{align*}
where $m(i,j)$ denotes the sample of the $j$-th sample from the $i$-th client.
Let $\sigma_{\min,*} := \sigma_{\min}(\bar{\Lambda}),$ and let $\sigma_{\max,*} := \sigma_{\max}(\bar{\Lambda})$. 
% Let $\tilde{\kappa} := \sigma_{\max,*}/\sigma️_{\min,*}$. 
Suppose that $L^0 \geq c \mathrm{polylog}(d,nL^0) \sigma_{\max,*} d k^2/(n \sigma_{\min,*}^2)$. Then, with probability at least $1 - (nL^0)^{-100}$,
we have that
\begin{align*}
    \mathrm{dist}(\mB^0,\mB^*)^2 \leq \tilde{O}\left( \frac{\sigma_{\max,*} k^2 d}{\sigma_{\min,*}^2 n L^0}\right).
\end{align*}
In particular, when the number of samples is uniform across the domains, we have that
\begin{align*}
    \mathrm{dist}(\mB^0,\mB^*)^2 \leq \tilde{O}\left( \frac{\kappa^4 k^2 d}{ n L^0}\right),
\end{align*}
where we recall that $\kappa := \bar{\sigma}_{\max,*}/\bar{\sigma}_{\min,*}$, and
\begin{align*}
   \bar{\sigma}_{\max,*} := \sigma_{\max}\left(\frac{1}{\sqrt{M}} W^* \right), \quad \bar{\sigma}_{\min,*} := \sigma_{\min}\left(\frac{1}{\sqrt{M}} W^* \right).
\end{align*}
\end{theorem}
\begin{proof}
We omit the proof since it is a slight variant of Theorem 3 in \cite{tripuraneni2021provable}. For completeness, note that in the case when the number of samples is uniform across the domains, some algebra shows that
\begin{align*}
    \mathrm{dist}(\mB^0,\mB^*)^2 \leq \tilde{O}\left( \frac{\kappa^2 k^2 d}{ \bar{\sigma}_{\min,*}^2 n L^0}\right).
\end{align*}
However, since $k/4 M \leq \norm*{W^*}_F^2 \leq k M \bar{\sigma}_{\max,*}^2,$
we have that
\begin{align*}
    \frac{1}{\bar{\sigma}_{\min,*}^2} &= \kappa^2 \frac{1}{\bar{\sigma}_{\max,*}^2} \leq 4 \kappa^2, 
\end{align*}
which proves the last statement in the theorem.

\end{proof}
\newpage

\section{Additional Experimental Results}
\subsection{Experiments on FairFace dataset for gender classification}
\label{sec:add_exp}
\setlength{\tabcolsep}{4pt}
\begin{table}[h]
\caption{Min, max and average test accuracy of gender classification across 7 domains ({\it race groups}) on FairFace with number of clients $n=5$, number of samples at each client $L_i=500$.}
\label{table:face_gender}
\centering
\resizebox{\textwidth}{!}{%
\begin{tabular}{l l ccc| ccc| ccc| ccc}
\toprule
\multirow{2}{*}{Task} & \multirow{2}{*}{Method} & \multicolumn{3}{c|}{$\alpha=0.1$} & \multicolumn{3}{c|}{$\alpha=0.5$}      & \multicolumn{3}{c|}{$\alpha=1$}      & \multicolumn{3}{c}{$\alpha=100$}              \\ 
  & & Max    & Min    & Avg    & Max           & Min  & Avg  & Max           & Min  & Avg  & Max           & Min  & Avg \\\midrule \midrule

 %%%%% gender %%%%% 

  \multirow{4}{*}{Gender} & FedAvg                   & \textbf{92.0}    & 71.7   &83.9    &89.8    &77.6  &84.5 & 91.0    &77.4  &84.2  &90.5 &77.1  &84.7           \\
& FedAvg + Multi-head           & 90.2   &  48.7  & 78.9   &   89.2    & 77.8 & 84.1 &   \textbf{91.6}        & 76.8 & 83.9 & 91.1 & 77.5 & 84.5 \\
& FedDAR-WA                     & 89.8   & 53.4   & 80.9   & \textbf{91.5} & 76.7 & 84.3 & 91.2 & 76.1 & 84.3 & 90.0         & 76.8 & 84.1         \\
& FedDAR-SA &
  \textbf{92.2} &
  \textbf{73.4} &
  \textbf{85.1} &
   \textbf{91.3}&
   \textbf{78.1}&
   \textbf{85.2}&
\textbf{91.4}&
   \textbf{78.2}&
   \textbf{85.1}&
  \textbf{92.2}
   & \textbf{78.1} & \textbf{85.6} \\ 
  \bottomrule
\end{tabular}%
}
\vspace{-8pt}

\end{table}
%\vspace{-5pt}
\setlength{\tabcolsep}{1.4pt}
We also conduct experiments for gender classification on FairFace with the same settings. The best representation dimension is $k=2$ for this task, probably due to the smaller diversity across the domains. We can see that the results shown in Table \ref{table:face_gender} have similar trend with the results in Table \ref{table:face_age}. 

\subsection{Experiments on digits dataset}
\setlength{\tabcolsep}{4pt}
\begin{table}[h]
\caption{Min, max and average test accuracy of digits classification across 5 domains with number of clients $n=5$, number of samples at each client $L_i=500$. }
\label{table:digits}
\centering
\resizebox{\textwidth}{!}{%
\begin{tabular}{l ccc| ccc| ccc| ccc}
\toprule
 \multirow{2}{*}{Method} & \multicolumn{3}{c|}{$\alpha=0.1$} & \multicolumn{3}{c|}{$\alpha=0.5$}      & \multicolumn{3}{c|}{$\alpha=1$}      & \multicolumn{3}{c}{$\alpha=100$}              \\ 
   & Max    & Min    & Avg    & Max           & Min  & Avg  & Max           & Min  & Avg  & Max           & Min  & Avg \\\midrule \midrule

    FedAvg                   & \textbf{97.1}    & \textbf{60.7}   &\textbf{80.6}    & \textbf{97.2}    & 64.3  & 81.7 &96.1  &\textbf{74.8}  &85.2  &96.8 &\textbf{71.0}  &85.1           \\
 FedAvg + Multi-head           & 94.3   &  26.5  & 55.9   &   94.3    & 44.8 & 68.3 &   94.1       & 56.7 & 74.6 & 95.0 & 52.3 & 74.5 \\
 FedDAR-WA                     &\textbf{97.3}   & 52.3   & 79.8   & \textbf{97.3} & \textbf{64.7} & \textbf{83.1} & \textbf{96.6} & 74.5 & \textbf{86.3} & \textbf{97.1}         & 70.6 & \textbf{86.3}         \\

  \bottomrule
\end{tabular}%
}
\vspace{-8pt}

\end{table}
%\vspace{-5pt}
\setlength{\tabcolsep}{1.4pt}
We perform additional experiments on digits dataset with five data domains with feature shift~\cite{li2021fedbn}. Details are described in the following paragraphs. From Table \ref{table:digits}, we can see that \our-WA outperform FedAvg consistently except the case where domain distributions are extremely heterogeneous ($\alpha=0.1$). In this case, each client tends to have data from only one domain. It is difficult for the proposed method to learn a good domain-specific head for the domain with the most different data (more obvious feature shift)  under this circumstance. For other levels of heterogeneity, although the min and max domain accuracies are similar between FedAvg and \our-WA, the average accuracies are improved as a result of domain-wise personalized model. One the other hand, without alternative update of the head and representation, FedAvg + Multi-head will overfit quickly. We don't include the results of \our-SA here because using representation dimension $k\geq 64$ causes numerical instability during head aggregation and failure to converge. While using representation dimension $k\leq 32$ leads to lower accuracy.

\paragraph{Datasets.} We use the same digits dataset containing five different data domains as~\cite{li2021fedbn}. Specifically, we use SVHN~\cite{netzer2011reading}, USPS~\cite{hull1994database}, SynthDigits~\cite{ganin2015unsupervised}, MNIST-M~\cite{ganin2015unsupervised} and MNIST~\cite{lecun1998gradient} as five data domains. Similarity to the experiments on FairFace datraset, the training data is divided into $n$ clients without duplication. Each client has a domain distribution $\bm{\pi_i} \sim Dir(\alpha \vp)$ sampled from a Dirichlet distribution.  

\paragraph{Implementation Details.} We adapt the codebase from~\cite{li2021fedbn}. A 6-layer CNN with 3 convolutional layers and 3 fully-connected layers is used, with the last layer as domain-specific head. We use SGD optimizer with learning rate $10^{-2}$ and cross-entropy loss. The batch size is set to $32$, and the total communication rounds is set to $100$. For each method, we first train the model for $10$ rounds with $1$ local epoch using FedAvg as warmup. The accuracy shown is the average over the last ten communication rounds. We repeat experiment for each setting three times with different random seeds and report the averages.

\subsection{Further experimental details}

\subsubsection{Synthetic Data}
For the synthetic data experiments, we adapt the code from~\cite{collins2021exploiting} and follow a similar protocol. The ground-truth matrices $\mW^*\in \bbR^{M\times k}$ and $\mB^*\in \bbR^{d\times k}$ are generated following the same way as~\cite{collins2021exploiting} by sampling each element from i.i.d. standard normal distribution and taking the QR factorization. The same $L$ samples are used for each client during the whole training process. Test samples are generated in the same way as the traning samples but without noise. For all the methods, models are initalized with ramdom Gaussian samples. We set $\alpha=0.4$ for experiments in Figure \ref{fig:synthetic_mse}.

\subsubsection{Real data with controlled distribution}

\paragraph{Implementation details.}
We use Imagenet\cite{deng2009imagenet} pre-trained ResNet-34~\cite{he2016deep} for all experiments on this dataset. All the methods are trained for $T=100$ communication rounds, with $20$ rounds of FedAvg as warmup. We use Adam optimizer with a learning rate of $1\times 10^{-4}$ for the first $60$ rounds and $1\times 10^{-5}$ for the last $40$ rounds. The images are resized to $224 \times 224$ with only random horizontal flip for augmentation.

Our evaluation metrics are the classification accuracy on the whole validation set of FairFace for each race group. We don't have extra local validation set to each client since we assume the data distribution within each domain is consistent across the clients. The numbers reported are the average over the final $10$ rounds of communication following the standard practice in~\cite{collins2021exploiting}, and the average of three independent runs with different random seeds. 

\subsubsection{Real data with real-World data distribution}

\paragraph{Dataset details.} The detailed statistics of the partial EXAM dataset is summarized in Table \ref{table:exam-summary}. The "Other" category includes American Indian or Alaska native, native Hawaiian or other Pacific islander and patients with more than one race or unknown race. $\geq$HFO \% means the percentage of cases with positive labels (receiving oxygen therapy higher or equal to high-flow oxygen with 72 hours).
\setlength{\tabcolsep}{4pt}
\begin{table}[h]

\caption{Data summary of the partial EXAM dataset used in our study. }
\label{table:exam-summary}
\centering
\resizebox{0.6\textwidth}{!}{%
\begin{tabular}{c|ccccc|c}
\toprule
Site   & White  & Black  & Asian   & Latino & Other  & $\geq$HFO \% \\ \midrule
Site-1 & 59.6\% & 10.0\% & 3.4\%   & 2.0\%  & 24.9\% & 12.4\%                \\
Site-2 & 75.0\% & 11.1\% & 2.8\%   & 0.6\%  & 10.5\% & 9.1\%                 \\
Site-3 & 46.5\% & 26.3\% & 4.2\%   & 7.0\%  & 16.0\% & 9.6\%                 \\
Site-4 & 71.4\% & 6.3\%  & 4.2\%   & 0.8\%  & 17.2\% & 11.4\%                \\
Site-5 & 44.0\% & 28.4\% & 1.6\%   & 6.3\%  & 19.8\% & 9.9\%                 \\
Site-6 & 0.0\%  & 0.0\%  & 100.0\% & 0.0\%  & 0.0\%  & 18.8\% \\        \bottomrule     
\end{tabular}%
}
\end{table}

\paragraph{Implementation details.}
We apply 5-fold cross validation. All the models are trained for $T=20$ communication rounds with Adam optimizer and a learning rate of $1\times 10^{-4}$. For each round we do $1$ local epoch for all the methods. For all the methods, the models are initialized with the same pretrained model as in \cite{dayan2021federated} without any warmup. For \our-SA and \our-WA, we excute 5 epochs of update for heads on each round, and set representation dimension $k=16$ for \our-SA. For FedRep,\our and FedPer. For LG-FedAvg, we treated the last fully-connected layer as the global parameters and all other layers as local representation. For FedMinMax, multiple local iterations are executed during each round instead of one step of GD for reasonable comparison. For FedProx we tuned $\mu$ among $\{0.05,0.1,0.25,0.5\}$ and used $\mu=0.1$. For the fine-tuning methods, we only fine-tune the global trained model locally with Adam optimizer and learning rate of $5e-5$ for 1 epoch since more epochs of fine-tuning leads to worse results.

The models are evaluated by aggregating predictions on the local validation sets then calculating the area under curve (AUC) for each domain. The average AUCs on local validation set of clients are also reported. The AUC shown is first averaged over the last five communication rounds, and then averaged over five runs of 5-fold cross validation.

\end{document}